\documentclass[12pt]{article}
\usepackage{amsmath}
\usepackage{graphicx,psfrag,epsf}
\usepackage{epsfig}
\usepackage{epstopdf}
\usepackage{enumerate}
\usepackage{natbib}
\usepackage{url} 
\usepackage{hyperref, bm}
\usepackage{amsmath,amssymb}
\usepackage{bbm}
\usepackage{xcolor}
\usepackage{caption}
\usepackage{array}
\usepackage{subcaption}
\usepackage{multirow}
\usepackage{amsfonts,amstext,amsmath,amssymb,amsthm}
\usepackage{mathtools}
\DeclarePairedDelimiter{\floor}{\lfloor}{\rfloor}
\DeclarePairedDelimiter{\ceil}{\lceil}{\rceil}
\DeclareMathOperator{\tr}{tr}
\newcolumntype{P}[1]{>{\centering\arraybackslash}p{#1}}

\newcommand{\bbf}{\bm{f}}
\newcommand{\bba}{\bm{a}}
\newcommand{\bbs}{\bm{s}}

\newcommand{\bbR}{\mathbb{R}}

\newcommand{\sign}{\mathop{\mathrm{sign}}}

\newtheorem{definition}{Definition}
\newtheorem{assumption}{Assumption}

\newtheorem{lemma}{Lemma}
\newtheorem{Example}{Example}

\newtheorem{theorem}{Theorem}

\newcommand{\arginf}{\mathop{\mathrm{arginf}}}
\newcommand{\argmin}{\mathop{\mathrm{argmin}}}

\newcommand\numberthis{\addtocounter{equation}{1}\tag{\theequation}}

\newcommand{\blind}{0}
\def\KL{\textcolor{red}}
\begin{document}

\def\spacingset#1{\renewcommand{\baselinestretch}%
{#1}\small\normalsize} \spacingset{1}


\if0\blind
{
  \title{\bf Semiparametric Regression for Spatial Data via Deep Learning}
\author{
	Kexuan Li \thanks{kexuan.li.77@gmail.com, Global Analytics and Data Sciences, Biogen, Cambridge, Massachusetts, US.},  	
	Jun Zhu \thanks{jzhu@stat.wisc.edu, Department of Statistics, University of Wisconsin-Madison. },
	Anthony R. Ives \thanks{arives@wisc.edu, Department of Integrative Biology, University of Wisconsin-Madison. }, \\
	Volker C.  Radeloff \thanks{radeloff@wisc.edu, Department of Forest and Wildlife Ecology, University of Wisconsin-Madison.}, and
	Fangfang Wang\thanks{fwang4@wpi.edu, Department of Mathematical Sciences, Worcester Polytechnic Institute. }
}

  \maketitle
} \fi

\if1\blind
{
  \bigskip
  \bigskip
  \bigskip
  \begin{center}
    {\LARGE\bf Title}
\end{center}
  \medskip
} \fi

\bigskip
\begin{abstract}
In this work, we propose a deep learning-based method to perform semiparametric regression analysis for spatially dependent data. To be specific, we use a sparsely connected deep neural network with rectified linear unit (ReLU) activation function to estimate the unknown regression function that describes the relationship between response and covariates in the presence of spatial dependence. Under some mild conditions, the estimator is proven to be consistent, and the rate of convergence is determined by three factors: (1) the architecture of neural network class, (2) the smoothness and (intrinsic) dimension of true mean function, and (3) the magnitude of spatial dependence. Our method can handle well large data set owing to the stochastic gradient descent optimization algorithm. Simulation studies on synthetic data are conducted to assess the finite sample performance, the results of which indicate that the proposed method is capable of picking up the intricate relationship between response and covariates. Finally, a real data analysis is provided to demonstrate the validity and effectiveness of the proposed method.
\end{abstract}

\noindent%

\spacingset{1.3}

\newpage
\section{Introduction}
With recent advances in remote sensing technology and geographical sciences, there has been a considerable interest in modeling spatially referenced data. The purpose of this paper is to develop new methodology that captures complex structures in such data via deep neural networks and Gaussian random fields. In addition, we provide a theoretical understanding of deep neural networks for spatially dependent data.

In recent years, deep neural network (DNN) has made a great breakthrough in many fields, such as computer vision \citep{Literature_Review_CV}, dynamics system \citep{ODE},  natural language processing \citep{Literature_Review_NLP}, drug discovery and toxicology \citep{Literature_Review_drug}, and variable selection \citep{DeepFS, DeepSurv}. Besides its successful applications, there has also been great progress on theoretical development of deep learning. \citet{DIVE} and \cite{schmidt-hieber} proved that the neural network estimator achieves the optimal (up to a logarithmic factor) minimax rate of convergence. \citet{liu2022optimal} further removed the logarithmic term and achieved the exact optimal nonparametric convergence rate. One of the appealing features of deep neural network is that it can circumvent the curse of dimensionality under some mild conditions.

Owing to the superior performance and theoretical guarantees of deep learning, applying deep learning to spatial data has also drawn much attention. For example, \citet{deep_spatial_2} fitted the integro-difference equation through convolutional neural networks and obtained probabilistic spatio-temporal forecasting. \citet{deep_spatial_4} constructed a deep probabilistic architecture to model nonstationary spatial processes using warping approach. \citet{deep_spatial_1} used variational autoencoder generative neural networks to analyze spatio-temporal point processes. \citet{deep_spatial_3} applied Bayesian deep neural network to spatial interpolation. However, there is a lack of theoretical understanding of the aforementioned work, which we will address in this paper.

In addition, we model spatial dependence by Gaussian random fields and develop model estimation with computational efficiency. Due to technological advances in data collecting process, the size of spatial datasets are massive and traditional statistical methods encounter two challenges. One challenge is the aggravated computational burden. To reduce computation cost, various methods have been developed, such as covariance tapering, fixed rank kriging, and Gaussian Markov random fields (see \citet{sun2012geostatistics} for a comprehensive review). The other challenge is data storage and data collection. Many spatial datasets are not only big, but are also generated by different sources or in an online fashion such that the observations are generated one-by-one. In both cases, we cannot process entire datasets at once.  To overcome these two challenges, \citet{SGD1951} proposed a computationally scalable algorithm called stochastic gradient descent (SGD) and achieved great success in many areas. Instead of evaluating the actual gradient based on an entire dataset, SGD estimates the gradient using only one observation which makes it computationally feasible with large scale data and streaming data. Its statistical inferential properties have also been studied by many researchers \citep{SGDsu2018, SGDliu2021}.

Before proceeding, it is essential to provide a brief overview of the classical approaches for formulating spatial dependence in a spatial process. For a spatial process, various approaches have been proposed to formulate the spatial dependence. In spatially varying coefcient (SVC) models, covariates may have different effect along with locations, which are referred to as ``spatial non-stationarity''. Thus, SVC models allows the coefficients of the covariates to change with the locations. In other words, the spatial dependence is entirely explained by the regressors, while the disturbances are independent (see e.g. \citet{SVCM_1, SVCM_2, SVCM_3, SVCM_4}). Another class of models are based on the spatial autoregressive (SAR) models of \citet{SAR_cliff}, where the information about spatial dependence is contained in the spatial weight matrix, and the response variable at each location is assumed to be affected by a weighted average of the dependent variables from other sampled units (see \citet{SAR_lee_2002, SAR_lee_2004}. In SAR models, the spatial weight matrix is assumed to be known, which is infeasible in practice, especially in large dataset. Obviously, both SVC and SAR models have some limitations. First, they assume the relationship between the response and the independent variable is linear and ignore complex interaction and nonlinear structure. Second, they both involve computing the inverse of an $n \times n$ matrix, where $n$ is the number of observations, which requires $O(n^3)$ time complexity and $O(n^2)$ memory complexity. Thus, the computational burden for both SVC and SAR model is substantial.

To meet these challenges, here we develop deep learning-based semiparametric regression for spatial data. Specifically, we use a sparsely connected feedforward neural network to fit the regression model, where the spatial dependence is captured by Gaussian random fields. By assuming a compositional structure on the regression function, the consistency of the neural network estimator is guaranteed. The advantages of the proposed method are fourfold. First, we do not assume any parametric functional form for the regression function, allowing the true mean function to be nonlinear or with complex interactions. This is an improvement over many of the existing parametric, semiparametric, or nonparametric approaches  {\citep{SVCM_1, SVCM_2, SVCM_3, SVCM_4, SAR_lee_2002, nonparametric_robinson2011, nonparametric_jenish2012, nonparametric_wavelet, nonparametric_lu2014, nonparametric_kurisu2019, nonparametric_kurisu2022}. Second, under some mild technical conditions, we show that the estimator is consistent. To the best of our knowledge, this is the first theoretical result in deep neural network for spatially dependent data. Third, the convergence rate is free of the input dimension, which means our estimator does not suffer from the curse of dimensionality. Finally, owing to the appealing properties of SGD, our method is feasible for large scale dataset and streaming data.

The remainder of the paper is organized as follows. Section \ref{model_estimator} formulates the statistical problem and presents the deep neural network estimator. The computational aspects and theoretical properties of the estimator are given in Section \ref{implementation_theory}. Section \ref{simulation} evaluates the finite-sample performance of the proposed estimator by a simulation study. We apply our method to a real-world dataset in Section \ref{real_data} and some concluding remarks are provided in Section \ref{conclusion}. Technical details are provided in Appendix.

\section{Model and Estimator} \label{model_estimator}

In this section, we first formulate the problem and then present the proposed estimator under a deep learning framework.

\subsection{Model Setup}
For a spatial domain of interest $\mathcal{S}$, we consider the following semiparametric spatial regression model:
\begin{equation} \label{eq_spatial_regression}
y(\bm{s})= f_0(\bm{x}(\bm{s})) + e_1(\bm{s}) + e_2(\bm{s}), \,\,\, \bm{s} \in \mathcal{S}
\end{equation}
where $f_0:[0, 1]^d\rightarrow\bbR$ is an unknown mean function of interest, $\bm{x}(\bm{s}) = (x_1(\bm{s}), \ldots, x_d(\bm{s}))^\top$ represents a $d$-dimensional vector of covariates at location $\bm{s}$ with $x_i(\bm{s}) \in [0,1]$,  $e_1(\bm{s})$ is a  mean zero Gaussian random field with covariance function $\gamma({\bm{s}, \bm{s}'})$, $\bbs, \bbs'\in \mathcal{S}$, 
and $e_2(\bm{s})$ is a spatial Gaussian white noise process with mean 0 and variance  $\sigma^2$.  
Furthermore, we assume that $e_1(\bm{s})$, $e_2(\bm{s})$, and $\bm{x}(\bm{s})$ are independent of each other. 
Thus the observation $y(\bbs)$ comprises three components: large-scale trend $f_0(\bm{x}(\bm{s}))$, small-scale spatial variation $e_1(\bbs)$, and measurement error $e_2(\bbs)$; see, for instance, \cite{cressie2008fixed}.

In the spatial statistics literature, it is popular to focus on predicting the hidden spatial process $y(\bm{s})^* = f_0(\bm{x}(\bm{s})) + e_1(\bm{s})$ using the observed information  \citep{cressie2008fixed}.
However, the primary interest of this paper is to estimate the large-scale trend $f_0(\bm{x}(\bm{s}))$, where the relationship between the hidden spatial process and the covariates could be complex in nature.  To capture such a complex relationship, we assume that $f_0$ is a composition of several functions inspired by neural networks characteristics \citep{schmidt-hieber}.  H\" older smoothness (see Definition \ref{holder.smoothness} in Appendix) is a commonly used smoothness assumption for regression function in nonparametric and semiparametric literature. Thus it is natural to assume the true mean function $f_0$ is a composition of H\" older smooth functions, which is formally stated in the following assumption.

\begin{assumption}\label{CompositionalStructure}
The function $f_0:\bbR^d\rightarrow\bbR$ has a compositional structure with parameters $(L_*, \bm{r}, \bm{\tilde{r}}, \bm{\beta}, \bm{a}, \bm{b}, \bm{C})$ where $L_*\in \mathbb{Z}_+$, $\bm{r}=(r_0, \ldots, r_{L_*+1})^\top$ $\in \mathbb{Z}_+^{L_*+2}$ with $r_0=d$ and $r_{L_*+1}=1$, $\bm{\tilde{r}}=(\tilde{r}_0,\ldots, \tilde{r}_{L_*})^\top \in \mathbb{Z}_+^{L_*+1}$, $\bm{\beta}=(\beta_0,\ldots, \beta_{L_*})^\top \in \mathbb{R}_+^{L_*+1}$, $\bm{a}=(a_0,\ldots, a_{L_*+1})^\top$, $\bm{b}=(b_0,\ldots, b_{L_*+1})^\top \in \mathbb{R}^{L_*+2}$, and $\bm{C}=(C_0,\ldots, C_{L_*})^\top \in \mathbb{R}_+^{L_*+1}$; that is,
\begin{equation*}
     f_0(\bm{z})=\bm{g}_{L_*}\circ\ldots\circ \bm{g}_1 \circ \bm{g}_0(\bm{z}),\quad\quad \textrm{ for } \bm{z} \in [a_0, b_0]^{r_0}
\end{equation*}
where $\mathbb{Z}_+, \mathbb{R}_+$ denote the sets of positive integers and positive real numbers, respectively, $\bm{g}_i=(g_{i,1},\ldots, g_{i,r_{i+1}})^\top: [a_i, b_i]^{r_i}\to [a_{i+1}, b_{i+1}]^{r_{i+1}}$ for some $|a_i|, |b_i|\leq C_i$ and the functions $g_{i,j}: [a_i, b_i]^{\tilde{r}_i} \to [a_{i+1}, b_{i+1}]$ are $(\beta_i, C_i)$-H\" older smooth only relying on $\tilde{r}_i$ variables and $\tilde{r}_i\le r_i$.
\end{assumption}

Without loss of generality, we assume $C_i>1$ in Assumption \ref{CompositionalStructure}.  The parameter $L_*$ refers to the total number of layers, i.e., the number of composite functions, $\bm{r}$ is the whole number of variables in each layer, whereas $\bm{\tilde{r}}$ is the number of  ``active" variables in each layer. The two parameter vectors $\bm{\beta}$ and $\bm{C}$ pertain to the H\" older smoothness in each layer, while $\bm{a}$ and $\bm{b}$ define the domain of $\bm{g}_i$ in the $i$th layer.  For the rest of the paper, we will use $\mathcal{CS}(L_*, \bm{r}, \bm{\tilde{r}}, \bm{\beta},\bm{a}, \bm{b}, \bm{C})$ to denote the class of functions that  have a compositional structure as specified in Assumption \ref{CompositionalStructure} with parameters $(L_*, \bm{r}, \bm{\tilde{r}}, \bm{\beta}, \bm{a}, \bm{b}, \bm{C})$.

It is worth mentioning that Assumption \ref{CompositionalStructure} is commonly adopted in deep learning literature; see \citet{IntroBauer}, \citet{schmidt-hieber}, \citet{Kohle2019}, \citet{DIVE}, \citet{ODE}, among others. This compositional structure covers a wide range of function classes including the generalized additive model. In Example 1 and Figure \ref{figure:example1}, we present an illustrative example to help readers understand the concept of compositional structure within the context of the generalized additive model.

\begin{Example}\label{example}
The generalized additive model is a generalized linear model with a linear predictor involving a sum of smooth functions of covariates \citep{GAM}. Suppose $f(\bm{z}) = \varphi(\sum_{i=1}^{d}h_i(z_i))$, where $\varphi(\cdot)$ is $(\beta_\varphi, C_\varphi)$-H\" older smooth and $h_{i}(\cdot)$ are $(\beta_h, C_h)$-H\" older smooth, for some $(\beta_\varphi, C_\varphi)$ and $(\beta_h, C_h)$. Clearly, $f(\bm{z})$ can be written as a composition of three functions $f(\bm{z})=\bm{g}_2\circ \bm{g}_1\circ \bm{g}_0(\bm{z})$ with $\bm{g}_0(z_1, \ldots, z_{d}) = (h_1(z_1),\ldots,h_d(z_{d}))^{\top}$, $\bm{g}_1(z_1, \ldots, z_{d}) =\sum_{i=1}^{d}z_i$, and $\bm{g}_2(z) = \varphi(z)$. Here, $L_*=2$, $\bm{r}=(d, d, 1, 1)^{\top}$, $\bm{\tilde{r}} = (d, d, 1)^{\top}$, and $\bm{\beta} = (\beta_h, \infty, \beta_\varphi)^{\top}$. 
\end{Example}

\begin{figure}[ht!]
\centering
\includegraphics[width=6in]{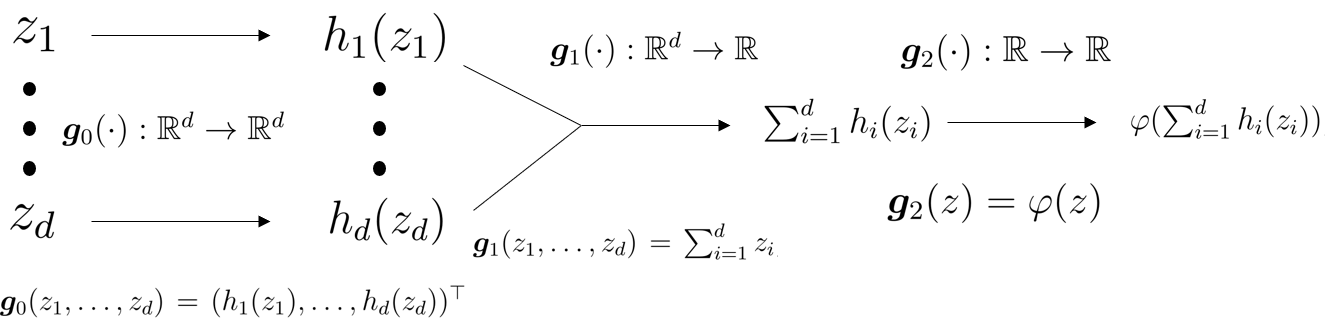}
\caption{Illustration of the compositional structure for the generalized additive model in Example 1.} \label{figure:example1}
\end{figure}

\subsection{Deep Neural Network (DNN) Estimator}

In this paper, we consider estimating the unknown function $f_0$ via a deep neural network owing to the complexity of $f_0$ and the flexibility of neural networks.  So before presenting our main results, we first briefly review the neural network terminologies pertaining to this work.

An activation function is a nonlinear function used to learn the complex pattern from  data. In this paper, we focus on the Rectified Linear Unit (ReLU) shifted activation function which is defined as $\sigma_{\bm{v}}(\bm{z})= (\sigma(z_1-v_1),\ldots,\sigma(z_d-v_d))^{\top}$,  where $\sigma(s) = \max\{0,s\}$ and $\bm{z} = (z_1,\ldots,z_d)^{\top} \in \bbR^d$.  ReLU activation function enjoys both theoretical and computational advantages.  The projection property $\sigma \circ \sigma = \sigma$ can facilitate the proof of consistency, while ReLU activation function can help avoid vanishing gradient problem. The ReLU feedforward neural network $\bbf(\bm{z}, W, v)$ is given by
\begin{equation}\label{fnn}
\bbf(\bm{z}, W, v) =W_{L}\sigma_{\bm{v}_{L}} \ldots W_1 \sigma_{\bm{v}_1} W_0 \bm{z}, \quad \bm{z}\in\bbR^{p_0},
\end{equation}
where $\{(W_0,\ldots,W_{L}): W_{l}\in \mathbb{R}^{p_{l+1}\times p_{l}}, 0\le l\le L\}$ is the collection of weight matrices, $\{(\bm{v}_1,\ldots,\bm{v}_{L}): \bm{v}_{l}\in \mathbb{R}^{p_l}, 1\le l\le L\}$ is the collection of so-called biases (in the neural network literature), and $\sigma_{\bm{v}_l}(\cdot)$, $1\le l\le L$, are the ReLU shifted activation functions. Here, $L$ measures the number of hidden layers, i.e., the length of the network, while $p_j$ is the number of units in each layer, i.e., the depth of the network. When using a ReLU feedforward neural network to estimate the regression problem \eqref{eq_spatial_regression}, we need to have $p_0 = d$ and $p_{L+1} = 1$; and the parameters that need to be estimated are the weight matrices $(W_j)_{j=0,\ldots,L}$ and the biases $(\bm{v}_j)_{j=1,\ldots,L}$.   

By definition,  a ReLU feedforward neural network can be written as a composition of simple nonlinear functions; that is,
\[
\bbf(\bm{z}, W, v) = \bm{g}_{L}\circ\ldots\circ \bm{g}_1 \circ \bm{g}_0(\bm{z}), \quad \bm{z}\in\bbR^{p_0},
\]
where $\bm{g}_i(\bm{z}) = W_{i}\sigma_{\bm{v}_{i}}(\bm{z})$, $\bm{z}\in\bbR^{p_i}$, $i=1, \ldots, L$, and  $\bm{g}_0(\bm{z})=W_0 \bm{z}, \bm{z}\in\bbR^{p_0}$. Unlike traditional function approximation theory where a complex function is considered as an infinite sum of simpler functions (such as Tayler series, Fourier Series, Chebyshev approximation, etc.), deep neural networks approximate a complex function via compositions, i.e., approximating the function by compositing simpler functions  \citep{IntroApproximation2, farrell2021deep, YAROTSKY2017103}.  Thus, a composite function can be well approximated by a feedforward neural network.  That is why we assume the true mean function $f_0$ has a compositional structure.

In practice, the length and depth of the networks can be extremely large, thereby easily causing overfitting. To overcome this problem, a common practice in deep learning is to randomly set some neurons to zero, which is called dropout. Therefore, it is natural to assume the network space is sparse and all the parameters are bounded by one, where the latter can be achieved by dividing all the weights by the maximum weight \citep{IntroBauer, schmidt-hieber, Kohle2019}.  As such, we consider the following sparse neural network class with bounded weights
\begin{equation}\label{class2}
\mathcal{F}(L, \bm{p}, \tau, F) = \left\{f \textrm{ is of form } \eqref{fnn}: \max_{j=0,\ldots,L} \|W_j\|_\infty + |\bm{v}_j|_\infty \leq 1, \sum_{j=0}^{L}(\|W_j\|_0 + |\bm{v}_j|_0) \leq \tau, \|f\|_\infty \leq F\right\},
\end{equation}
where $\bm{p}=(p_0, \ldots, p_{L+1})$ with $p_0 = d$ and $p_{L+1}=1$, and $\bm{v}_0$ is a vector of zeros.   This class of neural networks is also adopted in \citet{schmidt-hieber}, \citet{DIVE}, and \citet{ODE}.

Suppose that the process $y(\cdot)$ is observed at a finite number of spatial locations $\{\bm{s}_1, \ldots, \bm{s}_n\}$ in $\mathcal{S}$.
The desired DNN estimator of $f_0$ in Model (\ref{eq_spatial_regression}) is a sparse neural network in $\mathcal{F}(L, \bm{p}, \tau, F)$ with the smallest empirical risk;  that is,
\begin{equation} \label{FNN Estimator}
\widehat{f}_{\textrm{global}}(\widehat{W}, \widehat{v}) = \argmin_{f\in\mathcal{F}(L, \bm{p}, \tau, F)}n^{-1}\sum_{i=1}^{n}(y(\bm{s}_i) - f(\bm{x}(\bm{s}_i))^2.
\end{equation}
For simplicity,  we sometimes write $\widehat{f}_{\textrm{global}}$ for  $\widehat{f}_{\textrm{global}}(\widehat{W}, \widehat{v})$ if no confusion arises.

\section{Computation and Theoretical Results} \label{implementation_theory}
In this section, we describe the computational procedure used to optimize the objective function (\ref{FNN Estimator}) and present the main theoretical results.

\subsection{Computational Aspects} \label{Computational_Aspects}

Because (\ref{FNN Estimator}) does not have an exact solution,  we use a stochastic gradient descent (SGD)-based algorithm to optimize (\ref{FNN Estimator}). In contrast to a gradient descent algorithm which requires a full dataset to estimate gradients in each iteration, SGD or mini-batch gradient descent only needs an access to  a subset of observations during each update, which is capable of training relatively complex models for large datasets and computationally feasible with streaming data.  Albeit successful applications in machine learning and deep learning,  SGD still suffers from some potential problems. For example, the rate of convergence to the minima is slow; the performance is very sensitive to  tuning parameters. To circumvent these problems, various methods have been proposed, such as RMSprop and Adam \citep{Adam}. In this paper, we use Adam optimizer to solve (\ref{FNN Estimator}).

During the training process, there are many hyper-parameters to tune in our approach: the number of layers $L$, the number of neurons in each layer $\bm{p}$, the sparse parameter $\tau$, and the learning rate. These hyper-parameters play an important role in the learning process. However, it is challenging to determine the values of hyper-parameters without domain knowledge. In particular, it is challenging to control the sparse parameter $\tau$ directly in the training process. Thus, we add an $\ell_1$-regularization penalty to control the number of inactive neurons in the network. The idea of adding a sparse regularization to hidden layers in deep learning  is very common; see, for instance, \cite{scardapane2017group} and \cite{Lassonet}. In this paper, we use a $5$-fold  cross-validation to select tuning parameters.

\subsection{Theoretical Results}

Recall that the minimizer of (\ref{FNN Estimator}), $\widehat{f}_{\textrm{global}}$, is practically unattainable and we use an SGD-based algorithm to minimize the objective function (\ref{FNN Estimator}), which may converge to a local minimum.  The actual estimator obtained by minimizing (\ref{FNN Estimator}) is denoted by  $\widehat{f}_{\textrm{local}} \in \mathcal{F}(L, \bm{p}, \tau, F)$. We define the  difference between the expected empirical risks of $\widehat{f}_{\textrm{global}}$ and $\widehat{f}_{\textrm{local}}$ as
\begin{align} \label{Delta_n}
\Delta_n(\widehat{f}_{\textrm{local}}) &\doteq \bm{E}_{f_0} \left[ \frac{1}{n}\sum_{i=1}^{n}(y(\bm{s}_i) - \widehat{f}_{\textrm{local}}(\bm{x}(\bm{s}_i))^2 - \inf_{\tilde{f} \in \mathcal{F}(L, \bm{p}, \tau, F)}\frac{1}{n}\sum_{i=1}^{n}(y(\bm{s}_i) - \tilde{f}(\bm{x}(\bm{s}_i))^2 \right] \nonumber\\
&= \bm{E}_{f_0}\left[ \frac{1}{n}\sum_{i=1}^{n}(y(\bm{s}_i) - \widehat{f}_{\textrm{local}}(\bm{x}(\bm{s}_i))^2 - \frac{1}{n}\sum_{i=1}^{n}(y(\bm{s}_i) - \widehat{f}_{\textrm{global}}(\bm{x}(\bm{s}_i))^2\right],
\end{align}
where $\bm{E}_{f_0}$ stands for the expectation with respect to the true regression function $f_0$.
For any $\widehat{f}\in \mathcal{F}(L, \bm{p}, \tau, F)$, we consider the following estimation error: 
\begin{equation}\label{eqn.risk}
R_n(\widehat{f}, f_0) \doteq \bm{E}_{f_0}\left[ \frac{1}{n}\sum_{i=1}^{n} \big(\widehat{f}(\bm{x}(\bm{s}_i)) - f_0(\bm{x}(\bm{s}_i))\big)^2 \right].
\end{equation}

The oracle-type theorem below gives an upper bound for the estimation error.
\begin{theorem}\label{thm: oracle}
Suppose that the unknown true mean function $f_0$ in (\ref{eq_spatial_regression}) satisfies $\|f_0\|_\infty \le F$ for some $F \geq 1$. For any $\delta, \epsilon \in (0, 1]$ and $\widehat{f} \in \mathcal{F}(L, \bm{p}, \tau, F)$,   the following oracle inequality holds:
\begin{align*}
	R_n(\widehat{f}, f_0)  \lesssim&  (1+\varepsilon)\left( \inf_{\tilde{f} \in \mathcal{F}(L, \bm{p}, \tau, F)} \|\tilde{f}-f_0\|^2_\infty   + \zeta_{n, \varepsilon, \delta} + \Delta_n(\widehat{f}) \right) ,
\end{align*}
where
\begin{align*}
	\zeta_{n, \varepsilon, \delta} \asymp& \frac{1}{\varepsilon} \left[ \delta  \Big(n^{-1}\tr(\Gamma_n)  +  2\sqrt{n^{-1}\tr(\Gamma_n^2)} + 3\sigma \Big) + \frac{\tau}{n} \left(\log(L/\delta) + L\log\tau \right) (n^{-1}\tr(\Gamma_n^2) + \sigma^2+1)      \right],
\end{align*}
and $\Gamma_n = [\gamma(\bm{s}_i, \bm{s}_{i'})]_{1\leq i, i' \leq n}$.
\end{theorem}

The convergence rate in Theorem \ref{thm: oracle} is determined by three components. The first component $\inf_{\tilde{f} \in \mathcal{F}(L, \bm{p}, \tau, F)} \|\tilde{f}-f_0\|^2_\infty$ measures the distance between the neural network class $\mathcal{F}(L, \bm{p}, \tau, F)$ and $f_0$, i.e., the approximation error. The second term $\zeta_{n, \varepsilon, \delta}$ pertains to the estimation error, and $\Delta_n(\widehat{f})$ is owing to the difference between $\widehat{f}$ and the oracle neural network estimator $\widehat{f}_{\textrm{global}}$. It is worth noting that the upper bound in Theorem \ref{thm: oracle} does not depend on the network architecture parameter $\bm{p}$, i.e., the width of the network,  in that the network is sparse and its ``actual'' width is controlled by the sparsity parameter $\tau$. To see this, after removing all the inactive neurons, it is straightforward to show that
$\mathcal{F}(L, \bm{p}, \tau, F) = \mathcal{F}(L, (p_0, p_1\wedge\tau, \ldots, p_L\wedge\tau, p_{L+1}), \tau, F)$ \citep{schmidt-hieber}.

Next, we turn to the consistency of the DNN estimator $\widehat{f}_{\textrm{local}}$ for $f_0 \in \mathcal{CS}(L_*, \bm{r}, \bm{\tilde{r}}, \bm{\beta},\bm{a}, \bm{b}, \bm{C})$.  In nonparametric regression, the estimation convergence rate is heavily affected by the smoothness of the function.  Consider the class of  composite functions  $\mathcal{CS}(L_*, \bm{r}, \bm{\tilde{r}}, \bm{\beta},\bm{a}, \bm{b}, \bm{C})$.  Let $\beta_i^*=\beta_i\prod_{s=i+1}^{L_*}(\beta_s\wedge 1)$ for $i=0,\ldots, L_*$ and $i^*=\argmin_{0\leq i \leq L_*}\beta_i^*/\tilde{r}_i$, with the convention $\prod_{s=L_*+1}^{L_*}(\beta_s\wedge 1)=1$. Then $\beta^*=\beta_{i^*}^*$ and $r^*=\tilde{r}_{i^*}$ are known as the intrinsic smoothness and intrinsic dimension of $f \in \mathcal{CS}(L_*, \bm{r}, \bm{\tilde{r}}, \bm{\beta},\bm{a}, \bm{b}, \bm{C})$. Similar definitions could be found in literature \cite{nonparametric_kurisu2019, nakada2020adaptive, cloninger2021deep, kohler2022estimation, wang2023deep}. These quantities play an important role in controlling the convergence rate of the estimator.
To better understand $\beta_i^*$ and $i^*$, think about  the composite function from the $i$th to the last layer, i.e.,  $\bm{h}_i(\bm{z})=\bm{g}_{L_*}\circ\ldots\circ \bm{g}_{i+1} \circ \bm{g}_i(\bm{z}): [a_i, b_i]^{r_i}\to\mathbb{R}$; then $\beta_i^*$ can be viewed as the smoothness of $\bm{h}_i$ and $i^*$ is the layer of the least smoothness after rescaled by the respective number of ``active'' variables $\tilde{r}_i$, $i=0,\ldots, L_*$.
The following theorem establishes the consistency of $\widehat{f}_{\textrm{local}}$ as an estimator of $f_0$ and its convergence rate in the presence of spatial dependence.

\begin{theorem}\label{thm.main}
Suppose Assumption \ref{CompositionalStructure} is satisfied, i.e., $f_{0} \in \mathcal{CS}(L_*, \bm{r}, \bm{\tilde{r}}, \bm{\beta},\bm{a}, \bm{b}, \bm{C})$. Let $\widehat{f}_{\textrm{local}} \in \mathcal{F}(L, \bm{p}, \tau, F)$ be an estimator of $f_0$. Further assume that $F \geq \max_{i=0,\ldots,L^*}(C_i, 1)$, $N\doteq\min_{i=1,\ldots,L} p_i \geq 6\eta\max_{i=0,\ldots,L_*}(\beta_i+1)^{\tilde{r}_i} \vee (\tilde{C}_i+1)e^{\tilde{r}_i}$ where $\eta = \max_{i=0,\ldots, L_*}(r_{i+1}(\tilde{r}_i + \ceil{\beta_i}))$, and  $\tau\lesssim LN$.  Then we have
\[
R_n(\widehat{f}_{\textrm{local}}, f_0) \lesssim \varsigma_n,
\]
where
\[
 \varsigma _{n} \asymp (N2^{-L})^{2\prod_{l=1}^{L_*}\beta_l\wedge1} + N^{-\frac{2\beta^*}{r^*}} + \frac{(\tr(\Gamma_n^2) + n)(LN\log(Ln^2) +  L^2N \log(LN))}{n^2} + \Delta_n(\widehat{f}_{\textrm{local}}),
\]
and $\tilde{C}_i$ are constants only depending on $\bm{C}, \bm{a}, \bm{b}$, and $\Gamma_n = [\gamma(\bm{s}_i, \bm{s}_{i'})]_{1\leq i, i' \leq n}$.
\end{theorem}
The consistency of $\widehat{f}_{\textrm{local}}$ can be achieved by,  for instance,  letting $L\asymp\log(n), N\asymp n^{\frac{r^*}{2\beta^* + r^*}}, \tr(\Gamma_n^2) = o(n^{\frac{4\beta^*+r^*}{2\beta^*+r^*}}(\log n)^{-3})$, and $\Delta_n(\widehat{f}_{\textrm{local}}) =o(1)$, as a result of which $\varsigma _{n} \asymp n^{-\frac{2\beta^*}{2\beta^* + r^*}}(\log n)^3 + \Delta_n(\widehat{f}_{\textrm{local}}) = o(1)$.
As expected, the rate of convergence is affected by the intrinsic smoothness and intrinsic dimension of $\mathcal{CS}(L_*, \bm{r}, \bm{\tilde{r}}, \bm{\beta},\bm{a}, \bm{b}, \bm{C})$, the architecture of the neural network $\mathcal{F}(L, \bm{p}, \tau, F)$, and the magnitude of the spatial dependence.

\section{Simulation Study} \label{simulation}
In this section, we evaluate the finite sample performance of the proposed DNN estimator through a set of simulation studies. Two different simulation designs are considered, and for each design, we generate 100 independent data sets. In both settings, we use the same neural network architecture with length $L = 3$ and  width $N = 30$. Additionally, to prevent overfitting and control the sparsity of the network, we applied dropout with a probability of 0.2.


In the first design, the spatial domain of interest $\mathcal{S}$ is  in $\bbR$.  To be specific, we generate data from the following model
\begin{align*} \label{sim_1}
 y(s_i) = f_0(\bm{x}(s_i)) + e_1(s_i) + e_2(s_i), \,\,\, s_i \in  [0, D], ~ i = 1, 2, \ldots, n,
\end{align*}
and
\begin{align*}
 \bm{x}(s_i) = (x_{1}(s_i), \ldots, x_{5}(s_i))^\top = \left(s_i/D, \sin(10s_i/D), (s_i/D)^2, \exp(3s_i/D), (s_i/D+1)^{-1}\right)^\top \in \mathbb{R}^5,
\end{align*}
with the true mean function $f_0(\bm{x}(s_i))= \sum_{j=1}^5x_{j}(s_i)$.  The small-scale spatial variation $e_1(\cdot)$ is a zero-mean stationary and isotropic Gaussian process with an exponential covariance function $\gamma({s_i,s_j}) = \exp(-|s_i - s_j|/\rho)$ and the range parameter $\rho=0.1, 0.5, 1$. The measurement error $e_2(\cdot)$ is standard normal distributed and independent of $e_1(\cdot)$. It is worth mentioning that the covariates are location dependent.

We consider two different spatial domains: fixed domain and expanding domain.  For the fixed domain, $\mathcal{S}=[0, 1]$ is a fixed interval, i.e., $D=1$,  whereas for the expanding domain, the spatial domain $\mathcal{S}=[0, D]$ increases with the sample size $n$. The $n$ observations are equally spaced over the region. In both scenarios, we let $n=100, 200, 300$, and accordingly, $D=10, 20, 30$  in the expanding domain case.

\begin{figure}
\centering
\begin{subfigure}{.5\textwidth}
  \centering
  \includegraphics[width=.8\linewidth]{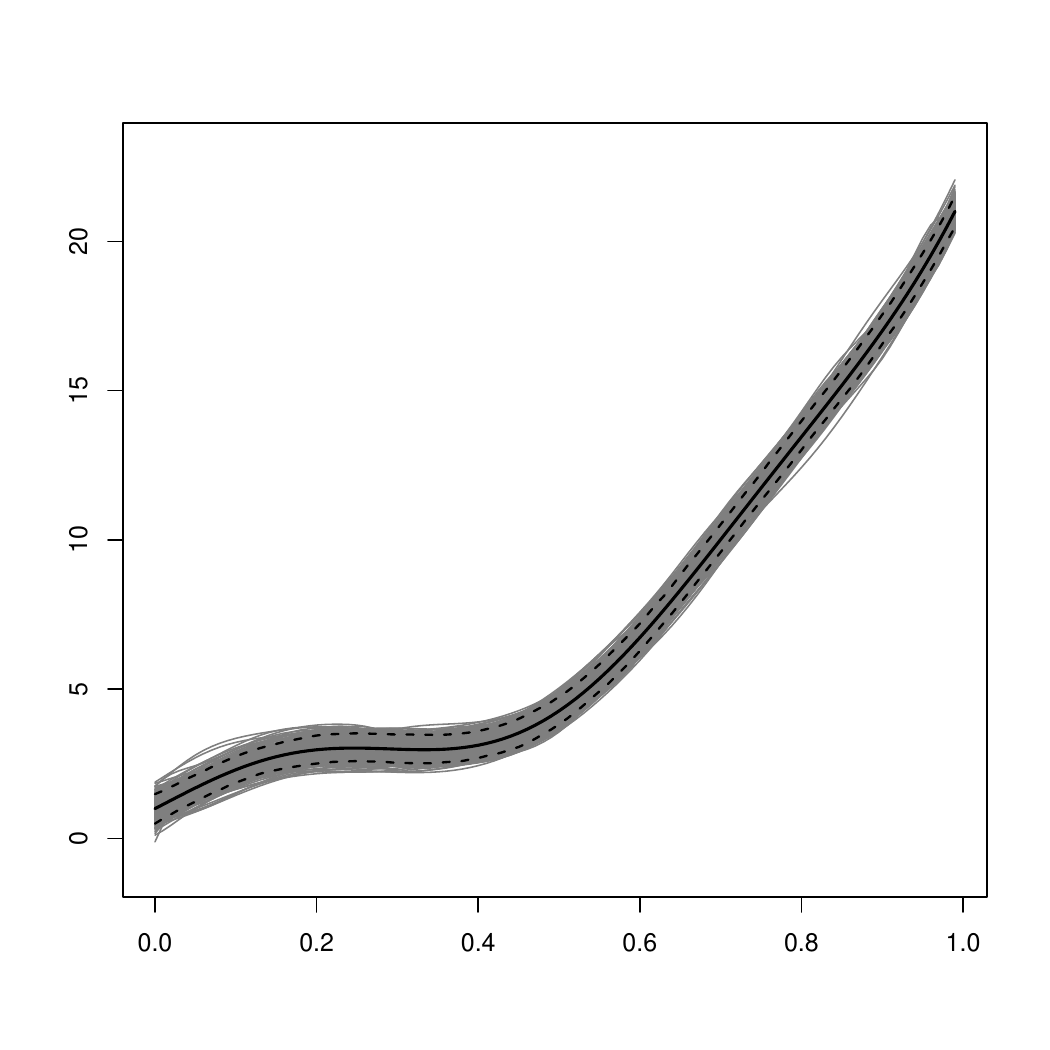}
  \caption{Fixed domain: $\mathcal{S}=[0, 1]$.}
  \label{fig:sim1_fixed}
\end{subfigure}%
\begin{subfigure}{.5\textwidth}
  \centering
  \includegraphics[width=.8\linewidth]{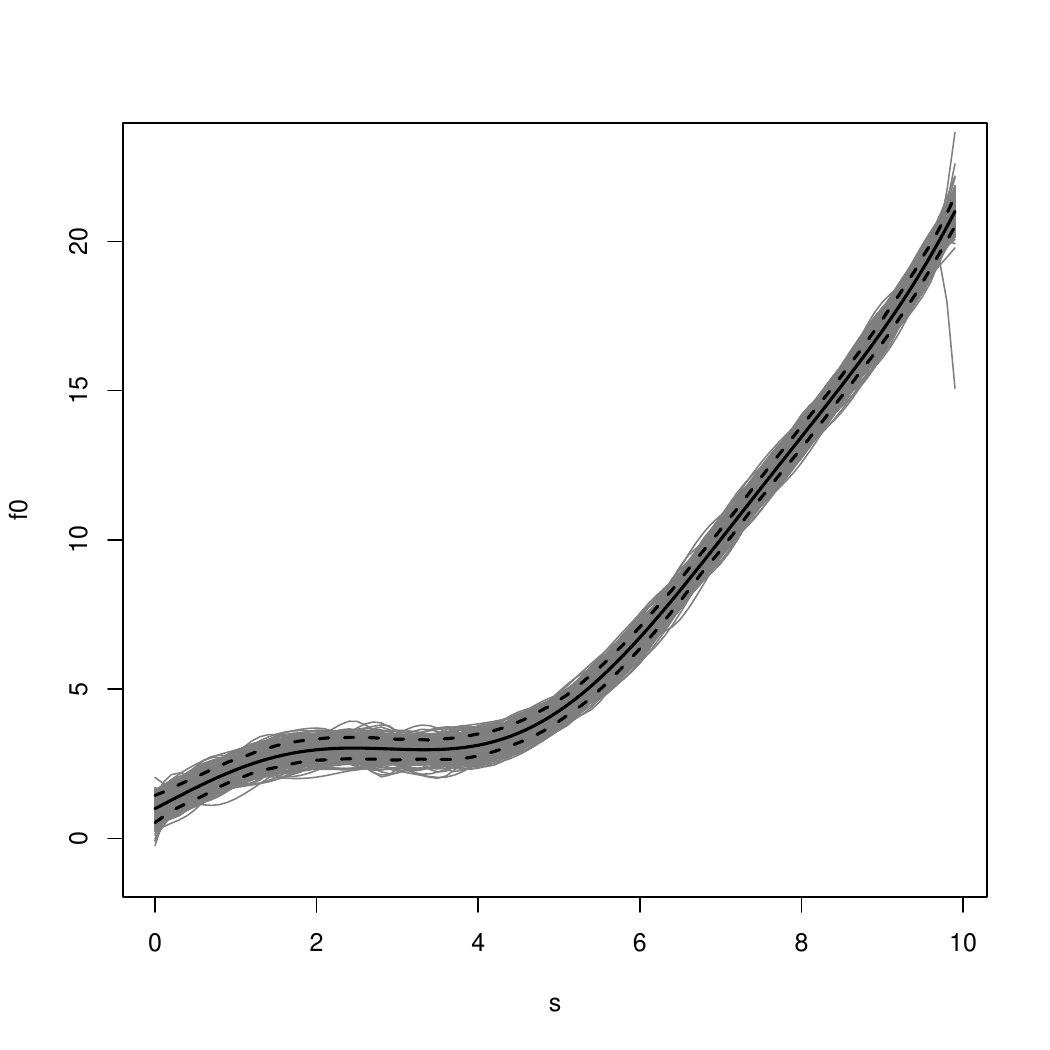}
  \caption{Expanding domain:  $\mathcal{S}=[0, 10]$.}
  \label{fig:sim1_expending}
\end{subfigure}
\caption{The estimated mean function and $95\%$ pointwise simulation intervals using our method in Simulation Design 1 with $n=100, \rho=0.5$. In both plots, the solid line is the true mean function, and the two dashed lines are the $95\%$ pointwise simulation intervals. The grey lines are the estimated mean functions from each replication.}
\label{fig:sim1}
\end{figure}


In the second design,   the mean function is defined on  $\bbR^2$, given by
\begin{align}
f_0(\bm{x}(\bm{s}_i)) =& \beta_1x_{1}(\bm{s}_i)x_{2}(\bm{s}_i) + \beta_2x_{2}(\bm{s}_i)^2\sin(x_{3}(\bm{s}_i))+\beta_3\exp(x_{4}(\bm{s}_i))\max(x_{5}(\bm{s}_i), 0) \nonumber \\
&+\frac{\beta_4}{\sign x_{4}(\bm{s}_i)(10+x_{5}(\bm{s}_i))} + \beta_5\tanh(x_{1}(\bm{s}_i)), \quad \bm{s}_i \in [0, D]^2,  ~ i = 1, \ldots, n,
\end{align}
where the coefficients $\beta_j$, $j=1, \ldots, 5$, are drawn from  $U(1, 2)$. The covariates at each location are generated from standard normal distributions with a cross-covariate correlation of 0.5 and the covariates at different locations are assumed to be independent. We further normalize each covariate to have zero mean and unit variance. The mean function $f_0$ is nonlinear, featuring interactions among the covariates.    We simulate $y(\bm{s}_i)$ according to \eqref{eq_spatial_regression} with  $e_1(\bm{s}_i)$ and $e_2(\bm{s}_i)$ similar to those in Design 1.  That is, $e_1(\bm{s}_i)$ is a zero-mean stationary and isotropic Gaussian process on $\bbR^2$ with an exponential covariance function $\gamma({\bm{s}_i,\bm{s}_j}) = \exp(-|\bm{s}_i - \bm{s}_j|/\rho)$ and $\rho=0.1, 0.5, 1$, and $e_2(\bm{s}_i) \sim N(0,1)$.

Similar to the first design, we consider two types of spatial domain: fixed domain, i.e., $D=1$ and expanding domain, i.e., $D = 10, 20, 30$.   In both cases, we have $n=100, 400, 900$, and all the locations are equally spaced over $[0, D]^2$.

\subsection{Estimating $f_0$ via other methods}

We also compare the proposed DNN estimator with various estimators in the literature.  The first estimator of $f_0$ is based on the Gaussian process-based spatially varying coefficient model (GP-SVC) which is given by
\begin{equation*}
y(\bm{s})= \beta_1(\bm{s})x_1(\bm{s}) + \ldots, +\beta_p(\bm{s})x_p(\bm{s}) + \epsilon, \,\,\, \epsilon \sim \mathcal{N}(0, \tau^2), \,\,\, \bm{s} \in \mathcal{S},
\end{equation*}
and the spatially varying coefficient $\beta_j(\cdot)$ is the sum of a fixed effect and a random effect. That is, $\beta_j(\bm{s}) = \mu_j + \eta_j(\bm{s})$, where $\mu_j$ is a non-random fixed effect and $\eta_j(\cdot)$ is a zero-mean Gaussian process with an isotropic covariance function $c(\cdot;\bm{\theta}_j)$. In this work, we use the popular Mat\'ern covariance function defined as
\[
c\left(r ; \rho, \nu, \sigma^2\right)=\sigma^2 \frac{2^{1-\nu}}{\Gamma(\nu)}\left(\sqrt{2 \nu} \frac{r}{\rho}\right)^\nu K_\nu\left(\sqrt{2 \nu} \frac{r}{\rho}\right),
\]
where $\rho>0$ is the range parameter, $\nu>0$ is the smoothness parameter, and $K_\nu(\cdot)$ is the modified Bessel function of second kind with order $\nu$. 

The second estimator is the Nadaraya-Watson (N-W) kernel estimator for spatially dependent data discussed in \citet{nonparametric_robinson2011}, which considers the following spatial regression model
    \[
    y(\bm{s}_i) = f_0(\bm{x}(\bm{s}_i)) + \sigma(\bm{x}(\bm{s}_i))V_i, \,\,\, V_i = \sum_{j=1}^\infty a_{ij}\epsilon_j, \,\,\, i=1, \ldots, n,
    \]
    where $f_0(\bm{x}):[0, 1]^d\rightarrow\bbR$ and $\sigma(\bm{x}):[0, 1]^d\rightarrow [0, \infty)$  are the mean and  variance functions, respectively,   $\epsilon_j$ are independent random variables with zero mean and unit variance, and $\sum_{j=1}^\infty a_{ij}^2 = 1$. \citet{nonparametric_robinson2011} introduces the following Nadaraya-Watson kernel estimator for $f_0$:
    $$
\hat{f}(\bm{x})=\frac{\hat{\nu}(\bm{x})}{\widehat{g}(\bm{x})},
$$
where
$$
\widehat{g}(\bm{x})=\frac{1}{n h_n^d} \sum_{i=1}^n K_i(\bm{x}), \quad \hat{\nu}(\bm{x})=\frac{1}{n h_h^d} \sum_{i=1}^n y_i K_i(\bm{x}),
$$
with
$$
K_i(\bm{x})=K\left(\frac{\bm{x}-\bm{x}(\bm{s}_i)}{h_n}\right),
$$
and $h_n$ is a scalar, positive bandwidth sequence satisfying $h_n \rightarrow 0$ as $n \rightarrow \infty$.

The third estimator of $f_0$ is based on the generalized additive model (GAM) mentioned in Example \ref{example}.  That is, we assume that \[
f_0(\bm{x}(\bm{s})) = \Psi\left(\sum_{j=1}^{d}g_j(x_j(\bm{s}))\right),
\]
where $g_j(\cdot): [0,1] \rightarrow\bbR$ and  $\Psi(\cdot):\bbR\rightarrow\bbR$ are some smooth functions. In this model,  spatial dependence is not assumed. The fourth estimator we consider is the linear regression (LR) model which assumes a linear relationship between the response variable and the predictor variables, without incorporating spatial dependence explicitly.

\subsection{Simulation Results}

To evaluate the performance of each method,  we generate additional $m=n/10$ observations at new locations, treated as a test set. Similar to \citet{SemiZhu}, we adopt mean squared estimation error (MSEE) and mean squared prediction error (MSPE) to evaluate the estimation and prediction performance, where MSEE and MSPE are defined as
$$\textrm{MSEE} = m^{-1}\sum_{i=1}^m(\widehat{f}(\bm{x}(s_i)) - f_0(\bm{x}(s_i)))^2,  \quad \textrm{MSPE} = m^{-1}\sum_{i=1}^m(\widehat{f}(\bm{x}(s_i)) - y(s_i))^2,$$
and $\widehat{f}(\bm{x}(s_i))$ is an estimator of $f_0(\bm{x}(s_i))$. The mean and standard deviation of $\textrm{MSEE}$ and $\textrm{MSPE}$ over the 100 independent replicates are summarized in Tables \ref{TableSim1_fixed_domain} -- \ref{TableSim2_expanding_domain}.

Tables \ref{TableSim1_fixed_domain} and \ref{TableSim1_expanding_domain} pertain to Simulation Design  1, for fixed and expanding domains, respectively. For each combination of the sample size $n$ and the spatial dependence $\rho$, we highlight the estimator in boldface that yields the smallest MSEE and MSPE.  Overall, GAM, GP-SVC, N-W, \KL{and LR} methods perform similar to each other.  The proposed DNN estimator produces a smaller estimation error and prediction error than the others in all cases except when $n=200$ and $\rho = 0.1$ in the fixed-domain case, GAM yields the smallest MSPE of 1.26.  But the MSPE produced by DNN is close. Despite that  spatial dependence has an adverse impact on the performance, when $n$ increases (and $D$ increases for the expanding-domain case), both estimation error and prediction error decrease as expected.

We depict in Figure \ref{fig:sim1} the estimated mean functions $\widehat{f}(\bm{x}(s))$ via our method with $n=100$ and $\rho=0.5$ from the 100 replications along with the $95\%$ pointwise confidence intervals for  both fixed and expanding domains. Here, the $95\%$ pointwise intervals are defined as
$$\left( 2^{-1}(\widehat{f}_{(2)}(\bm{x}(s_i))+\widehat{f}_{(3)}(\bm{x}(s_i))), 2^{-1}(\widehat{f}_{(97)}(\bm{x}(s_i))+\widehat{f}_{(98)}(\bm{x}(s_i))) \right), ~ i = 1, 2, \ldots, n,$$
where $\widehat{f}_{(k)}(\bm{x}(s_i))$ is the $k$th smallest value of $\{\widehat{f}_{[j]}(\bm{x}(s_i)): j=1, \ldots, 100\}$, and $\widehat{f}_{[j]}(\bm{x}(s_i))$ is the estimator of $f_0(\bm{x}(s_i))$ from the $j$th replicate.

Tables \ref{TableSim2_fixed_domain} and \ref{TableSim2_expanding_domain} report the  results for Simulation Design 2. For both fixed and expanding domains, our method performs the best among the \KL{five} methods and N-W comes next. This is mainly because \KL{LR}, GAM and GP-SVC treat $f_0$ to be linear and cannot handle complex interactions and nonlinear structures in $f_0$.

\begin{table}
\caption{Results of Simulation Design 1 with fixed domain: the averaged MSEE and MSPE over 100 replicates (with its standard deviation in parentheses) of various methods with different $n$ and $\rho$.} \label{TableSim1_fixed_domain}
\centering
\scalebox{0.8}{
\begin{tabular}{l|l|ll|ll|ll}
\hline
Fixed domain                &         &  $\rho=0.1$                         &      & $\rho=0.5$                        &      & $\rho=1$                         &      \\ \hline
\multicolumn{1}{c|}{$n$}     &   & \multicolumn{1}{l|}{MSEE} & MSPE & \multicolumn{1}{l|}{MSEE} & MSPE & \multicolumn{1}{l|}{MSEE} & MSPE \\ \hline
\multicolumn{1}{l|}{}      & GAM & \multicolumn{1}{l|}{0.92 (0.58)}  & 1.40 (0.69) & \multicolumn{1}{l|}{0.98 (0.78)} & 1.19 (0.40)    & \multicolumn{1}{l|}{1.05 (0.87)} & 1.17 (0.38)  \\ \cline{2-8}
\multicolumn{1}{c|}{$n=100$} & GP-SVC & \multicolumn{1}{l|}{0.87 (0.49)}& 1.37 (0.67)& \multicolumn{1}{l|}{0.92 (0.75)} & 1.15 (0.35)     & \multicolumn{1}{l|}{1.01 (0.82)}& 1.13 (0.36)     \\ \cline{2-8}
\multicolumn{1}{l|}{}      & N-W & \multicolumn{1}{l|}{0.89 (0.52)}  &  1.38 (0.67)& \multicolumn{1}{l|}{0.94 (0.76)} & 1.16 (0.38)     & \multicolumn{1}{l|}{1.03 (0.85)} & 1.15 (0.37)      \\ \cline{2-8}
\multicolumn{1}{l|}{}      & LR & \multicolumn{1}{l|}{0.93 (0.59)}  & 1.40 (0.69) & \multicolumn{1}{l|}{0.99 (0.80)} & 1.20 (0.39)    & \multicolumn{1}{l|}{1.06 (0.86)} & 1.18 (0.38) \\ \cline{2-8}
\multicolumn{1}{l|}{}      & DNN & \multicolumn{1}{l|}{{\bf 0.78} (0.41)} &  {\bf 1.32} (0.64) & \multicolumn{1}{l|}{{\bf 0.82} (0.71)} &{\bf 1.13} (0.35)     & \multicolumn{1}{l|}{{\bf 0.94} (0.79)}   & {\bf 1.10} (0.33)    \\ \hline
\multicolumn{1}{l|}{}      & GAM & \multicolumn{1}{l|}{0.87 (0.50)}  & {\bf1.26} (0.30) & \multicolumn{1}{l|}{0.93 (0.72)} & 1.12 (0.27)    & \multicolumn{1}{l|}{0.99 (0.77)} & 1.08 (0.24)     \\ \cline{2-8}
\multicolumn{1}{c|}{$n=200$} & GP-SVC & \multicolumn{1}{l|}{0.81 (0.42)}& 1.34 (0.37)& \multicolumn{1}{l|}{0.88 (0.74)} & 1.09 (0.29)     & \multicolumn{1}{l|}{0.95 (0.77)}& 1.09 (0.28)     \\ \cline{2-8}
\multicolumn{1}{l|}{}      & N-W & \multicolumn{1}{l|}{0.84 (0.38)}  &  1.33 (0.41)& \multicolumn{1}{l|}{0.91 (0.73)} & 1.10 (0.28)     & \multicolumn{1}{l|}{0.93 (0.75)} & 1.06 (0.25)      \\ \cline{2-8}
\multicolumn{1}{l|}{}      & LR & \multicolumn{1}{l|}{0.86 (0.50)}  & 1.28 (0.32) & \multicolumn{1}{l|}{0.95 (0.74)} & 1.14 (0.28)    & \multicolumn{1}{l|}{0.98 (0.76)} & 1.07 (0.24)     \\ \cline{2-8}
\multicolumn{1}{l|}{}      & DNN & \multicolumn{1}{l|}{{\bf0.69} (0.32)} &  1.27 (0.39) & \multicolumn{1}{l|}{{\bf0.71} (0.66)} & {\bf1.06} (0.26)     & \multicolumn{1}{l|}{{\bf0.78} (0.68)}   & {\bf 1.04} (0.29)    \\ \hline
\multicolumn{1}{l|}{}      & GAM & \multicolumn{1}{l|}{0.83 (0.47)}  & 1.19 (0.44) & \multicolumn{1}{l|}{0.88 (0.68)} & 1.09 (0.20)    & \multicolumn{1}{l|}{0.96 (0.66)} & 1.05 (0.19)     \\ \cline{2-8}
\multicolumn{1}{c|}{$n=300$} & GP-SVC & \multicolumn{1}{l|}{0.77 (0.38)}& 1.15 (0.37)& \multicolumn{1}{l|}{0.86 (0.71)} & 1.06 (0.21)     & \multicolumn{1}{l|}{0.91 (0.64)}& 1.05 (0.18)     \\ \cline{2-8}
\multicolumn{1}{l|}{}      & N-W & \multicolumn{1}{l|}{0.80 (0.36)}  &  1.13 (0.40)& \multicolumn{1}{l|}{0.88 (0.70)} & 1.07 (0.24)     & \multicolumn{1}{l|}{0.92 (0.66)} & 1.06 (0.21)      \\ \cline{2-8}
\multicolumn{1}{l|}{}      & LR & \multicolumn{1}{l|}{0.82 (0.46)}  & 1.17 (0.42) & \multicolumn{1}{l|}{0.87 (0.67)} & 1.10 (0.22)    & \multicolumn{1}{l|}{0.98 (0.69)} & 1.06 (0.20)     \\ \cline{2-8}
\multicolumn{1}{l|}{}      & DNN & \multicolumn{1}{l|}{{\bf0.58} (0.27)} &  {\bf1.07} (0.34) & \multicolumn{1}{l|}{{\bf0.63} (0.55)} &{\bf 1.01} (0.22)     & \multicolumn{1}{l|}{{\bf0.69} (0.55)}   &  {\bf1.01} (0.25)    \\ \hline
\end{tabular}
}
\end{table}

\begin{table}
\caption{Results of Simulation Design  1 with expanding domain (i.e., $D=10, 20, 30$): the averaged MSEE and MSPE over 100 replicates (with its standard deviation in parentheses) of various methods with different $n$ and $\rho$.} \label{TableSim1_expanding_domain}
\centering
\scalebox{0.8}{
\begin{tabular}{l|l|ll|ll|ll}
\hline
Expanding domain                &         &  $\rho=0.1$                         &      & $\rho=0.5$                        &      & $\rho=1$                         &      \\ \hline
\multicolumn{1}{l|}{$n$}     &  & \multicolumn{1}{l|}{MSEE} & MSPE & \multicolumn{1}{l|}{MSEE} & MSPE & \multicolumn{1}{l|}{MSEE} & MSPE \\ \hline
\multicolumn{1}{l|}{}      & GAM & \multicolumn{1}{l|}{0.35 (0.23)}  & 2.06 (0.67) & \multicolumn{1}{l|}{0.82 (0.38)} & 1.60 (0.59)    & \multicolumn{1}{l|}{0.98 (0.59)} & 1.42 (0.31)  \\ \cline{2-8}
\multicolumn{1}{l|}{$n=100, D=10$} & GP-SVC & \multicolumn{1}{l|}{0.33 (0.22)}  & 2.01 (0.61) & \multicolumn{1}{l|}{0.78 (0.36)} & 1.53 (0.55)    & \multicolumn{1}{l|}{0.94 (0.54)}   & 1.37 (0.29)     \\ \cline{2-8}
\multicolumn{1}{l|}{}      & N-W & \multicolumn{1}{l|}{0.38 (0.26)}  & 2.03 (0.65)  & \multicolumn{1}{l|}{0.81 (0.40)} & 1.57 (0.58)     & \multicolumn{1}{l|}{0.96 (0.57)} & 1.39 (0.29)     \\ \cline{2-8}
\multicolumn{1}{l|}{}      & LR & \multicolumn{1}{l|}{0.34(0.22)}  & 2.05 (0.66)  & \multicolumn{1}{l|}{0.81 (0.36)} & 1.58 (0.58)    & \multicolumn{1}{l|}{0.97 (0.58)} & 1.41 (0.30)     \\ \cline{2-8}
\multicolumn{1}{l|}{}      & DNN & \multicolumn{1}{l|}{ {\bf 0.26} (0.19)} & {\bf1.93} (0.55)    & \multicolumn{1}{l|}{{\bf 0.64} (0.37)} & {\bf1.44} (0.55)     & \multicolumn{1}{l|}{{\bf0.76} (0.44)} & {\bf 1.22} (0.26)    \\ \hline
\multicolumn{1}{l|}{}      & GAM & \multicolumn{1}{l|}{0.21 (0.14)}  & 1.91 (0.58) & \multicolumn{1}{l|}{0.66 (0.34)} & 1.51 (0.36)    & \multicolumn{1}{l|}{0.85 (0.44)} & 1.39 (0.39)     \\ \cline{2-8}
\multicolumn{1}{l|}{$n=200, D=20$} & GP-SVC & \multicolumn{1}{l|}{0.18 (0.14)}  & 1.89 (0.54)  & \multicolumn{1}{l|}{0.61 (0.33)} & 1.48 (0.39)     & \multicolumn{1}{l|}{0.81 (0.40)}   & 1.36 (0.41)     \\ \cline{2-8}
\multicolumn{1}{l|}{}      & N-W & \multicolumn{1}{l|}{0.20 (0.17)}     & 1.93 (0.61)   & \multicolumn{1}{l|}{0.63 (0.36)}     &  1.47 (0.37) & \multicolumn{1}{l|}{0.88 (0.48)}  & 1.40 (0.43)     \\ \cline{2-8}
\multicolumn{1}{l|}{}      & LR & \multicolumn{1}{l|}{0.20 (0.13)}  & 1.90 (0.56) & \multicolumn{1}{l|}{0.67 (0.35)} & 1.52 (0.35)    & \multicolumn{1}{l|}{0.86 (0.45)} & 1.39 (0.40)     \\ \cline{2-8}
\multicolumn{1}{l|}{}      & DNN & \multicolumn{1}{l|}{{\bf0.14} (0.11)}     & {\bf1.82} (0.55)   & \multicolumn{1}{l|}{{\bf0.43} (0.28)}    & {\bf1.33} (0.32)     & \multicolumn{1}{l|}{{\bf0.61} (0.37)}  & {\bf 1.29} (0.38)     \\ \hline
\multicolumn{1}{l|}{}      & GAM & \multicolumn{1}{l|}{0.11 (0.07)}  & 1.72 (0.39) & \multicolumn{1}{l|}{0.51 (0.29)} & 1.40 (0.31)    & \multicolumn{1}{l|}{0.70 (0.31)} & 1.30 (0.26)     \\ \cline{2-8}
\multicolumn{1}{l|}{$n=300, D=30$} & GP-SVC & \multicolumn{1}{l|}{0.13 (0.10)}     & 1.76 (0.41)     & \multicolumn{1}{l|}{0.56 (0.33)}     &    1.43 (0.36)  & \multicolumn{1}{l|}{0.74 (0.37)}     & 1.34 (0.30)     \\ \cline{2-8}
\multicolumn{1}{l|}{}      & N-W & \multicolumn{1}{l|}{0.13 (0.07)}     & 1.77 (0.44)     & \multicolumn{1}{l|}{0.53 (0.30)} & 1.44 (0.39)     & \multicolumn{1}{l|}{0.69 (0.33)}     & 1.29 (0.27)     \\ \cline{2-8}
\multicolumn{1}{l|}{}      & LR & \multicolumn{1}{l|}{0.12 (0.07)}  & 1.74 (0.40) & \multicolumn{1}{l|}{0.52 (0.31)} & 1.42 (0.33)    & \multicolumn{1}{l|}{0.72 (0.33)} & 1.32 (0.28)     \\ \cline{2-8}
\multicolumn{1}{l|}{}      & DNN & \multicolumn{1}{l|}{{\bf0.07} (0.09)}     &  {\bf1.63} (0.38)   & \multicolumn{1}{l|}{{\bf0.31} (0.23)}   & {\bf1.22} (0.23)     & \multicolumn{1}{l|}{{\bf0.43} (0.31)}  & {\bf1.10} (0.19)     \\ \hline
\end{tabular}
}
\end{table}

\begin{table}
\caption{Results of Simulation Design  2 with fixed domain: the averaged MSEE and MSPE over 100 replicates (with its standard deviation in parentheses) of various methods with different $n$ and $\rho$.} \label{TableSim2_fixed_domain}
\centering
\scalebox{0.8}{
\begin{tabular}{l|l|ll|ll|ll}
\hline
fixed-domain                &         &  $\rho=0.1$                         &      & $\rho=0.5$                        &      & $\rho=1$                         &      \\ \hline
\multicolumn{1}{l|}{$n$}     &  & \multicolumn{1}{l|}{MSEE} & MSPE & \multicolumn{1}{l|}{MSEE} & MSPE & \multicolumn{1}{l|}{MSEE} & MSPE \\ \hline
\multicolumn{1}{l|}{}      & GAM & \multicolumn{1}{l|}{0.87 (0.92)}  & 2.81 (1.19) & \multicolumn{1}{l|}{1.10 (2.40)} & 2.40 (1.53)    & \multicolumn{1}{l|}{1.23 (1.06)} & 2.16 (1.10)  \\ \cline{2-8}
\multicolumn{1}{l|}{$n=100$} & GP-SVC & \multicolumn{1}{l|}{0.93 (0.99)}  & 2.93 (1.23) & \multicolumn{1}{l|}{1.23 (2.54)}    & 2.59 (1.67)  & \multicolumn{1}{l|}{1.53 (1.33)} & 2.37 (1.26)  \\ \cline{2-8}
\multicolumn{1}{l|}{}      & N-W & \multicolumn{1}{l|}{0.73 (0.81)}  & 2.70 (1.15) & \multicolumn{1}{l|}{1.01 (2.16))}   & 2.30 (1.38)   & \multicolumn{1}{l|}{1.09 (1.01)}     &  2.09 (0.98)  \\ \cline{2-8}
\multicolumn{1}{l|}{}      & LR & \multicolumn{1}{l|}{1.16 (0.99)}  & 3.02 (1.25) & \multicolumn{1}{l|}{1.46 (2.62)} & 2.82 (1.76)    & \multicolumn{1}{l|}{1.88 (1.44)} & 2.66 (1.39)  \\ \cline{2-8}
\multicolumn{1}{l|}{}      & DNN & \multicolumn{1}{l|}{{\bf0.51} (0.60)}   &{\bf 2.27} (1.09)     & \multicolumn{1}{l|}{{\bf0.74} (1.11)}     & {\bf1.90} (1.02)  & \multicolumn{1}{l|}{{\bf0.83} (1.19)}     & {\bf1.99} (1.17)     \\ \hline
\multicolumn{1}{l|}{}      & GAM & \multicolumn{1}{l|}{0.44 (0.52)}  & 2.38 (0.58) & \multicolumn{1}{l|}{0.75 (0.65)} & 1.89 (0.42)    & \multicolumn{1}{l|}{0.92 (0.92)} & 1.69 (0.47)     \\ \cline{2-8}
\multicolumn{1}{l|}{$n=400$} & GP-SVC & \multicolumn{1}{l|}{0.50 (0.59)}     & 2.43 (0.66)   & \multicolumn{1}{l|}{0.82 (0.77)}     & 1.94 (0.47)     & \multicolumn{1}{l|}{1.00 (0.96)}   & 1.80 (0.53)     \\ \cline{2-8}
\multicolumn{1}{l|}{}      & N-W & \multicolumn{1}{l|}{0.39 (0.44)}     & 1.99 (0.58)   & \multicolumn{1}{l|}{0.68 (0.70)}     &  1.73 (0.41)    & \multicolumn{1}{l|}{0.83 (0.84))}     &  1.56 (0.41)    \\ \cline{2-8}
\multicolumn{1}{l|}{}      & LR & \multicolumn{1}{l|}{0.66 (0.52)}  & 2.67 (0.71) & \multicolumn{1}{l|}{0.95 (0.79)} & 2.05 (0.55)    & \multicolumn{1}{l|}{1.11 (0.97)} & 1.88 (0.57)     \\ \cline{2-8}
\multicolumn{1}{l|}{}      & DNN & \multicolumn{1}{l|}{{\bf0.22} (0.39)}     &  {\bf1.87} (0.49)    & \multicolumn{1}{l|}{{\bf 0.54} (0.61)}     & {\bf 1.57} (0.37)     & \multicolumn{1}{l|}{{\bf 0.68} (0.71)}     &  {\bf 1.41} (0.37)    \\ \hline
\multicolumn{1}{l|}{}      & GAM & \multicolumn{1}{l|}{0.31 (0.40)}  & 2.25 (0.53) & \multicolumn{1}{l|}{0.59 (0.53)} & 1.81 (0.36)    & \multicolumn{1}{l|}{0.80 (0.79)} & 1.65 (0.36)     \\ \cline{2-8}
\multicolumn{1}{l|}{$n=900$} & GP-SVC & \multicolumn{1}{l|}{0.38 (0.44)}     &  2.29 (0.58)    & \multicolumn{1}{l|}{0.66 (0.60)}     &   1.88 (0.39)   & \multicolumn{1}{l|}{0.88 (0.83)}     &  1.73 (0.40)    \\ \cline{2-8}
\multicolumn{1}{l|}{}      & N-W & \multicolumn{1}{l|}{0.25 (0.34)}     & 1.86 (0.49)     & \multicolumn{1}{l|}{0.51 (0.46)}     &   1.70 (0.31)   & \multicolumn{1}{l|}{0.71 (0.72))}     & 1.52 (0.34)     \\ \cline{2-8}
\multicolumn{1}{l|}{}      & LR & \multicolumn{1}{l|}{0.49 (0.48)}  & 2.33 (0.66) & \multicolumn{1}{l|}{0.88 (0.73)} & 1.95 (0.44)    & \multicolumn{1}{l|}{0.98 (0.86)} & 1.85 (0.42)    \\ \cline{2-8}
\multicolumn{1}{l|}{}      & DNN & \multicolumn{1}{l|}{{\bf 0.19} (0.27)}     & {\bf 1.70} ( 0.42)    & \multicolumn{1}{l|}{{\bf 0.28} (0.33)}     & {\bf 1.49} (0.29)     & \multicolumn{1}{l|}{{\bf 0.57} (0.59)}     & {\bf 1.33} (0.34)     \\ \hline
\end{tabular}
}
\end{table}

\begin{table}
\caption{Results of Simulation Design  2 with expanding domain ($D=10, 20, 30$): the averaged MSEE and MSPE over 100 replicates (with its standard deviation in parentheses) of various methods with different $n$ and $\rho$.} \label{TableSim2_expanding_domain}
\centering
\scalebox{0.8}{
\begin{tabular}{l|l|ll|ll|ll}
\hline
fixed-domain                &         &  $\rho=0.1$                         &      & $\rho=0.5$                        &      & $\rho=1$                         &      \\ \hline
\multicolumn{1}{l|}{$n$}     &  & \multicolumn{1}{l|}{MSEE} & MSPE & \multicolumn{1}{l|}{MSEE} & MSPE & \multicolumn{1}{l|}{MSEE} & MSPE \\ \hline
\multicolumn{1}{l|}{}      & GAM & \multicolumn{1}{l|}{0.75 (0.81)}  & 2.88 (1.04) & \multicolumn{1}{l|}{0.81 (0.65)} & 2.72 (0.94)    & \multicolumn{1}{l|}{0.90 (0.76)} & 2.65 (0.88)  \\ \cline{2-8}
\multicolumn{1}{l|}{$n=100, D=10$} & GP-SVC & \multicolumn{1}{l|}{0.84 (0.88)}     & 2.93 (1.11)     & \multicolumn{1}{l|}{0.89 (0.90)}     &  2.80 (0.97)     & \multicolumn{1}{l|}{0.96 (0.83))}     & 2.71 (0.91)     \\ \cline{2-8}
\multicolumn{1}{l|}{}      & N-W & \multicolumn{1}{l|}{0.66 (0.70)}     & 2.71 (1.00)     & \multicolumn{1}{l|}{0.71 (0.62)}     & 2.66 (0.90)     & \multicolumn{1}{l|}{0.82 (0.71))}     & 2.59 (0.81)     \\ \cline{2-8}
\multicolumn{1}{l|}{}      & LR & \multicolumn{1}{l|}{0.97 (0.91)}  & 2.95 (1.14) & \multicolumn{1}{l|}{1.06 (0.92)} & 2.92 (0.98)    & \multicolumn{1}{l|}{1.15 (0.86)} & 2.75 (0.95)  \\ \cline{2-8}
\multicolumn{1}{l|}{}      & DNN & \multicolumn{1}{l|}{{\bf 0.44} (0.49)}     &  {\bf  2.15} (1.02)   & \multicolumn{1}{l|}{{\bf 0.60} (1.03)}     &  {\bf 1.81} (0.99)    & \multicolumn{1}{l|}{{\bf 0.69} (1.07)}     & {\bf 1.76} (0.94)     \\ \hline
\multicolumn{1}{l|}{}      & GAM & \multicolumn{1}{l|}{0.32 (0.25)}  & 2.49 (0.44) & \multicolumn{1}{l|}{0.35 (0.24)} & 2.39 (0.46)    & \multicolumn{1}{l|}{0.40 (0.33)} & 2.32 (0.51)     \\ \cline{2-8}
\multicolumn{1}{l|}{$n=400, D=20$} & GP-SVC & \multicolumn{1}{l|}{0.40 (0.31)}     & 2.55 (0.48)     & \multicolumn{1}{l|}{0.49 (0.29)}   & 2.44 (0.50)     & \multicolumn{1}{l|}{0.54 (0.37))}     & 2.40 (0.55)     \\ \cline{2-8}
\multicolumn{1}{l|}{}      & N-W & \multicolumn{1}{l|}{0.28 (0.23)}     & 2.35 (0.41)     & \multicolumn{1}{l|}{0.31 (0.20)}     &  2.33 (0.41)    & \multicolumn{1}{l|}{0.35 (0.30)}     & 2.27 (0.47)     \\ \cline{2-8}
\multicolumn{1}{l|}{}      & LR & \multicolumn{1}{l|}{0.56 (0.35)} & 2.60 (0.54) & \multicolumn{1}{l|}{0.63 (0.33)} & 2.49 (0.49)   & \multicolumn{1}{l|}{0.71 (0.41)} & 2.46 (0.53)     \\ \cline{2-8}
\multicolumn{1}{l|}{}      & DNN & \multicolumn{1}{l|}{{\bf 0.18} (0.20)}     &  {\bf 2.20} (0.38)    & \multicolumn{1}{l|}{{\bf 0.24} (0.21)}     & {\bf 2.29} (0.38)     & \multicolumn{1}{l|}{{\bf 0.29} (0.24))}     & {\bf 2.20} (0.41)     \\ \hline
\multicolumn{1}{l|}{}      & GAM & \multicolumn{1}{l|}{0.24 (0.29)}  & 2.28 (0.33) & \multicolumn{1}{l|}{0.27 (0.16)} & 2.25 (0.27)    & \multicolumn{1}{l|}{0.31 (0.17)} & 2.26 (0.25)     \\ \cline{2-8}
\multicolumn{1}{l|}{$n=900, D=30$} & GP-SVC & \multicolumn{1}{l|}{0.31 (0.32))}     &  2.31 (0.35)    & \multicolumn{1}{l|}{0.37 (0.21)}   & 2.17 (0.25)     & \multicolumn{1}{l|}{0.41 (0.20)}     & 2.29 (0.28)     \\ \cline{2-8}
\multicolumn{1}{l|}{}      & N-W & \multicolumn{1}{l|}{0.21 (0.30)}     & 1.82 (0.46)     & \multicolumn{1}{l|}{0.26 (0.21)}     &    1.61 (0.28)  & \multicolumn{1}{l|}{0.29 (0.16)}     & 1.50 (0.30)     \\ \cline{2-8}
\multicolumn{1}{l|}{}      & LR & \multicolumn{1}{l|}{0.42 (0.35)}  & 2.37 (0.38) & \multicolumn{1}{l|}{0.53 (0.24)} & 2.33 (0.29)    & \multicolumn{1}{l|}{0.62 (0.23)} & 2.30 (0.31)     \\ \cline{2-8}
\multicolumn{1}{l|}{}      & DNN & \multicolumn{1}{l|}{{\bf 0.16} (0.22))}     & {\bf 1.71} (0.30))     & \multicolumn{1}{l|}{{\bf 0.19} (0.20)}     &  {\bf 1.58} (0.25)    & \multicolumn{1}{l|}{{\bf 0.23} (0.15))}      & {\bf 1.45} (0.28)     \\ \hline
\end{tabular}
}
\end{table}

\section{Data Example} \label{real_data}

In this section, we use the proposed DNN method to analyze California Housing data that are publicly available from the website \url{https://www.dcc.fc.up.pt/~ltorgo/Regression/cal_housing.html}. After removing  missing values, the dataset contains housing price information from $n = 20433$ block groups in California from the 1990 census, where a block group on average includes 1425.5 individuals living in a geographically compact area. To be specific, the dataset comprises median house values and six covariates of interest:  the median age of a house,  the total number of rooms, the total number of bedrooms,  population, the total number of households, and   the median income for households.
Figure \ref{figure:Cal_housing_hist} displays the histograms of the six covariates, from which one can observe that the covariates are all right skewed except for the median age of a house. Thus, we first apply the logarithm transform to the five covariates and then use min-max normalization to rescale all the six covariates so that the data are within the range $[0,1]$.

Figure \ref{figure:Cal_housing_scatter_covariates} shows the spatial distribution of the five log transformed  covariates (i.e., the total number of rooms, the total number of bedrooms,  population, the total number of households, and the median income for households) and the median age of a house.  We also depict in  Figure \ref{figure:Cal_housing_scatter} (the top panel) the map of the median house values in California.  The data exhibit a clear geographical pattern. Home values in the coastal region, especially the San Francisco Bay Area and South Coast, are higher than the other regions. Areas of high home values are always associated with high household income, dense population, large home size, and large household, which are clustered in the coastal region and Central Valley. Our objective is to explore the intricate relationship between the median house value and the six covariates by taking into account their spatial autocorrelation.

\begin{figure}[ht!]
  \centering
  \includegraphics[width=5in]{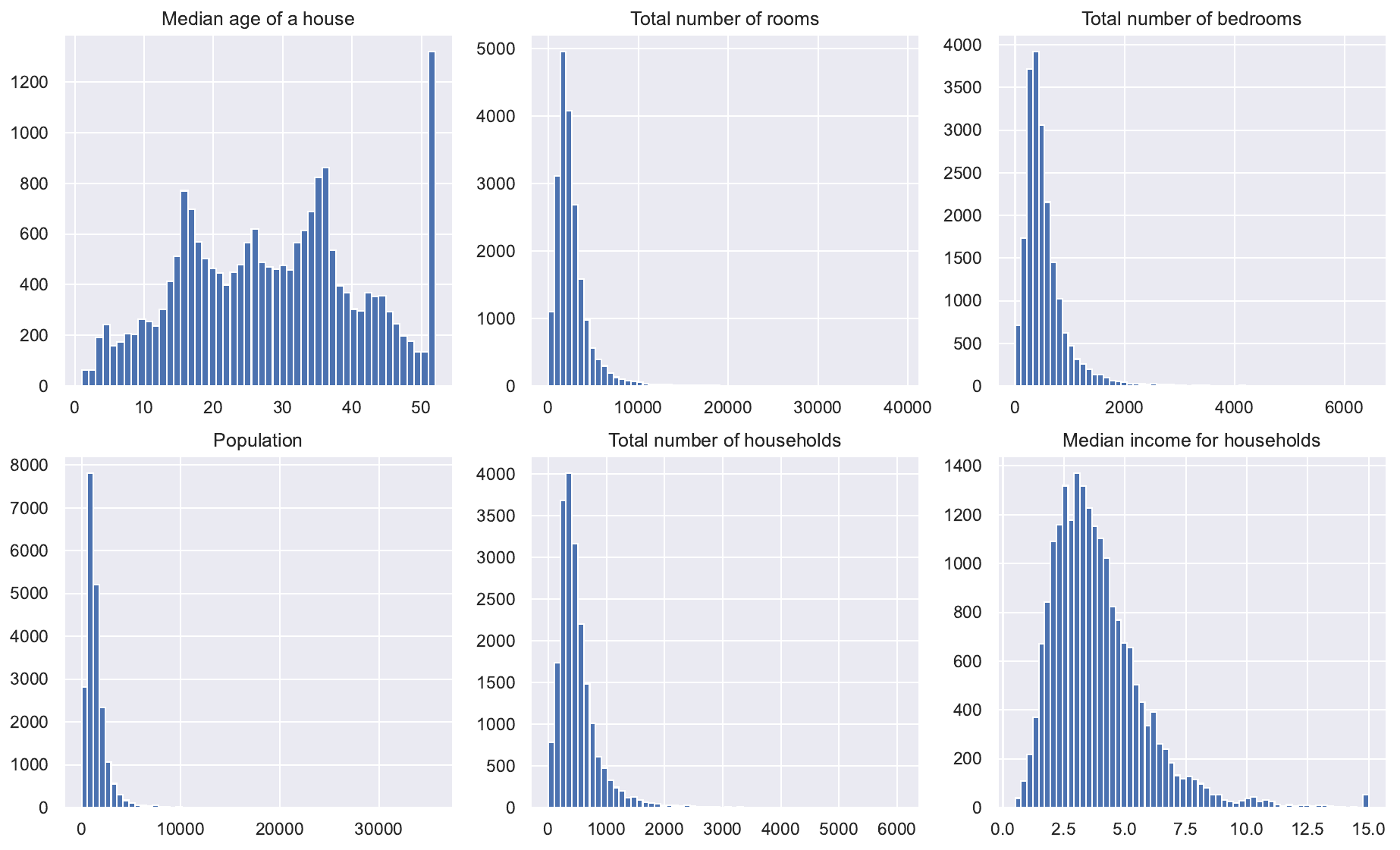}
  \caption{Histograms of six covariates in California housing data example.}  \label{figure:Cal_housing_hist}
\end{figure}

\begin{figure}[ht!]
  \centering
  \includegraphics[width=6in]{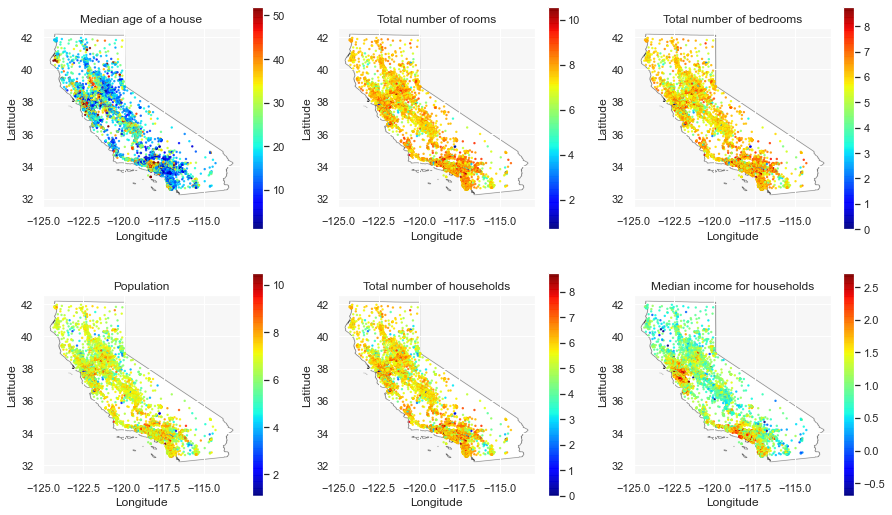}
  \caption{The map of six covariates in California housing data example.}  \label{figure:Cal_housing_scatter_covariates}
\end{figure}

\begin{figure}[ht!]
  \centering
  \includegraphics[width=5in]{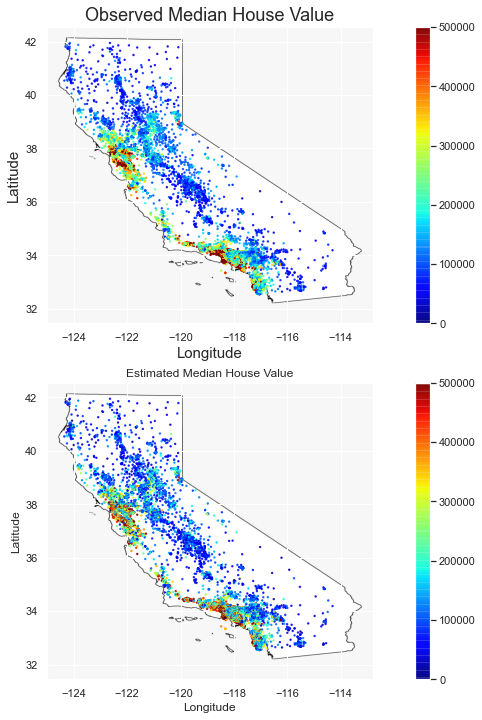}
  \caption{The top panel is the  map of 20433 observations and the corresponding median house value in California housing data example. The bottom panel is the estimated median house value.}  \label{figure:Cal_housing_scatter}
\end{figure}

Same as the simulation study, we estimate the mean function $f_0(\cdot)$ via four methods: DNN, GAM, GP-SVC, and N-W. We use the same neural network architecture as the simulation study, i.e., the length and width equal $L = 3$, $N = 30$, respectively, and dropout rate is set as 0.2 to avoid overfitting. To assess their performance, we compute the out-of-sample prediction error measured by MSPE based on 10-fold cross-validation, and the results are summarized in Table \ref{Tab:Real_Data}.  Consistent with the observations in the simulation study, the proposed DNN method yields a much more accurate prediction than the others. The bottom panel of Figure \ref{figure:Cal_housing_scatter} shows the estimated median house value using the DNN estimator, which exhibits a similar geographical pattern to the observations.

\begin{table}[]
\centering
\caption{Summary of the mean squared prediction error in California housing data example. }
\label{Tab:Real_Data}
\begin{tabular}{|c|c|c|c|c|c|}
\hline
Methods        & GAM & GP-SVC & N-W & LR& DNN  \\ \hline
MSPE ($\times 10^4$)   & 4.74    & 4.23       & 4.05      & 5.32      &  3.41       \\ \hline
\end{tabular}
\end{table}

\section{Conclusion} \label{conclusion}
In this study, we have ventured into the realm of regression analysis for spatially dependent data utilizing deep learning. Through meticulous consideration of the intricate interplay of spatial autocorrelation, our estimator has demonstrated its consistency. An intriguing path for future exploration unfolds in the domain of large-scale datasets. While the present study adeptly showcases the effectiveness of our approach via rigorous simulations, its true potential unfurls when confronted with the intricacies of real-world data. An illustrative instance is illuminated by the remote sensing data expounded in Ma et al.'s work \cite{ma2022scalable}. This dataset weaves together a tapestry of spatially referenced variables, collectively contributing to the intricate fabric of underlying relationships. The deployment of our model upon such expansive and diverse datasets holds the pledge of unveiling concealed patterns and amplifying comprehension of underlying processes.

Furthermore, our methodology's scope extends beyond predictive modeling, revealing a multitude of versatile applications. As expounded earlier, scenarios entailing interpolation challenges and gap-filling predicaments in remote sensing, alongside endeavors aimed at assessing the relative significance of covariates in ecological investigations, stand to derive substantial benefits from our model's nuanced capabilities. The vista of extending our framework to encompass spatiotemporal prediction broadens even further, beckoning exploration into more intricate, dynamic, and temporally evolving phenomena. For a more comprehensive review, refer to Wikle et al.'s comprehensive work \cite{wikle2019spatio}.

In conclusion, the journey into the domain of regression analysis for spatially dependent data endures as a continuous and captivating odyssey. As we confront the unique intricacies presented by expansive datasets and diverse problem domains, our proposed methodology emerges as a steadfast companion. Its role extends beyond mere contribution, shaping the advancement of statistical methodologies for spatial analysis.

\clearpage
\bibliographystyle{apalike}
\bibliography{ref}

\newpage
\newpage
\section{Appendix}

\subsection{Notation and Definition}
In this paper all vectors are column vectors, unless otherwise stated. Let $\|\bm{v}\|_2^2 = \bm{v}^{\top}\bm{v}$ for any vector $\bm{v}$, and $\|f\|_2 = \sqrt{\int f(x)^2dx}$ be the $L_2$ norm of a real-valued function $f(x)$. For two positive sequences ${a_n}$ and ${b_n}$, we write $a_n \lesssim b_n$ if there exists a positive constant $c$ such that $a_n \leq cb_n$ for all $n$, and $a_n \asymp b_n$ if $c^{-1}a_n \leq b_n \leq ca_n$  for some constant $c > 1$ and a sufficiently large $n$. Suppose that $\bm{x}=(x_1, \ldots, x_d)^{\top}$ is a $d$-dimensional vector. Let $|\bm{x}| = (|x_1|, \ldots, |x_d|)^{\top}$, $|\bm{x}|_\infty=\max_{i=1,\ldots,d}|x_i|$, and $|\bm{x}|_0=\sum_{i=1}^d \mathbbm{1}(x_i\neq0)$. For two $d-$dimensional vectors $\bm{x}$ and $\bm{y}$, we write $\bm{x} \lesssim \bm{y}$ if $x_i \lesssim y_i$ for $i=1,\ldots,d$. Let $\floor{x}$ be the largest number less than $x$ and $\ceil{x}$ be the smallest number greater than $x$. For a matrix $A=(a_{ij})$, let $\|A\|_\infty = \max_{ij}|a_{ij}|$ the max norm of $A$,    $\|A\|_0$ be the number of non-zero entries of $A$. Define $\|f\|_\infty$ as the sup-norm of a real-valued function $f$. We use $a\wedge b$ to represent the minimum of two numbers $a$ and $b$, while $a\vee b$ is the maximum of $a$ and $b$.



\begin{definition}[H\" older smoothness] \label{holder.smoothness}
A function $g: \mathbb{R}^{r_0} \to \mathbb{R}$ is said to be $(\beta, C)$-H\" older smooth for some positive constants $\beta$ and $C$, if for every $\bm{\gamma}=(\gamma_1, \ldots, \gamma_{r_0})\in \mathbb{N}^{r_0}$, the following two conditions hold:
\begin{eqnarray}
\sup_{\bm{x}\in
\mathbb{R}^{r_0}}\bigg|\frac{\partial^{\kappa}g}{\partial x_1^{\gamma_1}\ldots \partial x_{r_0}^{\gamma_{r_0}}}(\bm{x})\bigg|\leq C,
\quad     \textrm{ for }\kappa\leq \floor{\beta},\nonumber
\end{eqnarray}
and
\begin{eqnarray}
     \bigg|\frac{\partial^{\kappa}g}{\partial
x_1^{\gamma_1}\ldots \partial
x_{r_0}^{\gamma_{r_0}}}(\bm{x})-\frac{\partial^\kappa g}{\partial
x_1^{\gamma_1}\ldots \partial
x_{r_0}^{\gamma_{r_0}}}(\widetilde{\bm{x}})\bigg|\leq
C\|\bm{x}-\widetilde{\bm{x}}\|_2^{\beta-\floor{\beta}}, \quad     \textrm{ for } \kappa=\floor{\beta} \textrm{, } \bm{x},\widetilde{\bm{x}}\in
\mathbb{R}^{r_0},\nonumber
\end{eqnarray}
where $\kappa = \sum_{i=1}^{r_0}\gamma_{i}$. Moreover, we say $g$ is $(\infty, C)$-H\" older smooth if $g$ is $(\beta, C)$-H\" older smooth for all $\beta>0$.
\end{definition}


\subsection{Proof of Theorem \ref{thm: oracle}}

The proof of Theorem \ref{thm: oracle} requires a preliminary lemma. First, we define the $\delta$-cover of a function space $\mathcal{F}$ as a set $\widetilde{\mathcal{F}} \subset \mathcal{F}$ satisfying that following property: for any $f \in \mathcal{F}$, there exists a $\widetilde{f} \in \widetilde{\mathcal{F}}$ such that $\| \widetilde{f} - f\|_{\infty} \leq \delta$. Next, we define the $\delta$-covering number of $\mathcal{F}$ as
\[
\mathcal{N}(\delta, \mathcal{F}, \| \cdot \|_{\infty}) \doteq \min\{|\widetilde{\mathcal{F}}|: \widetilde{\mathcal{F}} \textrm{ is a $\delta$-cover of $\mathcal{F}$} \},
\]
where $|\mathcal{A}|$ means the number of distinct elements in set $\mathcal{A}$. In other words, $\mathcal{N}(\delta, \mathcal{F}, \| \cdot \|_{\infty})$ is the minimal number of $\| \cdot \|_{\infty}$-balls with radius $\delta$ that covers $\mathcal{F}$.   

\begin{lemma} \label{lemma: oracle}
Suppose that $f_0$ is the unknown true mean function in (\ref{eq_spatial_regression}). Let $\mathcal{F}$ be a function class such that $\{f_0\} \cup \mathcal{F} \subset \{f: [0, 1]^d \rightarrow [-F, F]\}$ for some $F \geq 1$. Then for all $\delta, \epsilon \in (0, 1]$ and $\widehat{f} \in \mathcal{F}$, the following inequality holds:
\begin{align*}
R_n(\widehat{f}, f_0)  \le&  (1+\varepsilon)\left( \inf_{\tilde{f}\in\mathcal{F}} R_n(\tilde{f}, f_0) + \Delta_n(\widehat{f})  + 2\delta\Big(n^{-1}\tr(\Gamma_n)  +   2\sqrt{n^{-1}\tr(\Gamma_n^2)} + 3\sigma \Big) \right) \\
&+ (1+\varepsilon) \frac{2F^2}{n\varepsilon}(3\log\mathcal{N} + 1)(n^{-1}\tr(\Gamma_n^2) + \sigma^2+1) ,
\end{align*}
where  $\mathcal{N} = \mathcal{N}(\delta, \mathcal{F}, \|\cdot\|_{\infty})$.
\end{lemma}

\begin{proof}
Let $\Delta_n = \Delta_n(\widehat{f})$. For any fixed $ f \in \mathcal{F}$, we have $\bm{E}_{f_0}\Big[n^{-1}\sum_{i=1}^{n}(y(\bm{s}_i) -  f(\bm{x}(\bm{s}_i))^2 - \inf_{\tilde f \in \mathcal{F}}n^{-1}\sum_{i=1}^{n}(y(\bm{s}_i) - \tilde f(\bm{x}(\bm{s}_i))^2\Big]\geq0$. Therefore,
\begin{align*}
&\bm{E}_{f_0}\Big[\frac{1}{n}\sum_{i=1}^{n}(y(\bm{s}_i) - \widehat{f}(\bm{x}(\bm{s}_i))^2\Big]\\
\leq & \bm{E}_{f _0}\Big[\frac{1}{n}\sum_{i=1}^{n}(y(\bm{s}_i) - \widehat{f}(\bm{x}(\bm{s}_i))^2 +\frac{1}{n}\sum_{i=1}^{n}(y(\bm{s}_i) -  f(\bm{x}(\bm{s}_i))^2 - \inf_{\tilde f \in \mathcal{F}}\frac{1}{n}\sum_{i=1}^{n}(y(\bm{s}_i) - \tilde f(\bm{x}(\bm{s}_i))^2\Big]\\
=& \bm{E}_{f _0}\Big[\frac{1}{n}\sum_{i=1}^{n}(y(\bm{s}_i) - {f}(\bm{x}(\bm{s}_i)))^2\Big] + \Delta_n.
\end{align*}
Let $\epsilon_i = e_1(\bm{s}_i)+e_2(\bm{s}_i)$.  Furthermore, we have
\begin{align*}
R_n(\widehat{f}, f_0) &= \frac{1}{n}\sum_{i=1}^n\bm{E}_{f_0}\big[\big(\widehat{f}(\bm{x}(\bm{s}_i)) - f_0(\bm{x}(\bm{s}_i))\big)^2\big] \\
&=\frac{1}{n}\sum_{i=1}^n \bm{E}_{f_0}\big[\big(\widehat{f}(\bm{x}(\bm{s}_i)) - y(\bm{s}_i) + y(\bm{s}_i) - f_0(\bm{x}(\bm{s}_i))\big)^2\big]\\
&= \frac{1}{n}\sum_{i=1}^n\bm{E}_{f_0}\big[\big(\widehat{f}(\bm{x}(\bm{s}_i)) - y(\bm{s}_i)\big)^2 + \epsilon_i^2+ 2\big(\widehat{f}(x(\bm{s}_i))-y(\bm{s}_i)\big)\epsilon_i\big]\\
&\leq \frac{1}{n}\sum_{i=1}^n\bm{E}_{f _0}\Big[(y(\bm{s}_i) - {f}(\bm{x}(\bm{s}_i)))^2 - \epsilon_i^2+ 2\widehat{f}(x(\bm{s}_i))\epsilon_i \Big] + \Delta_n\\
&= \frac{1}{n}\sum_{i=1}^n{\bm{E}_{f _0}\big[(y(\bm{s}_i) - f_0(\bm{x}(\bm{s}_i)) + f_0(\bm{x}(\bm{s}_i)) - {f}(\bm{x}(\bm{s}_i)))^2 - \epsilon_i^2+ 2\widehat{f}(x(\bm{s}_i))\epsilon_i \big] + \Delta_n}\\
&= \frac{1}{n}\sum_{i=1}^n\bm{E}_{f _0}\big[(f_0(\bm{x}(\bm{s}_i)) - {f}(\bm{x}(\bm{s}_i)))^2 + 2\widehat{f}(\bm{x}(\bm{s}_i))\epsilon_i \big] + \Delta_n . \numberthis \label{eqn: bnd R_n}
\end{align*}

Next, we will find an upper bound for $\bm{E}_{f _0}\big[\frac{2}{n}\sum_{i=1}^n\widehat{f}(x(\bm{s}_i))\epsilon_i \big]$. By the definition of the $\delta$-cover of a function space $\mathcal{F}$ and the $\delta$-covering number, we denote the centers of the balls by $f_j$, $j=1, 2, \ldots, \mathcal{N}$; and there exists $f_{j*}$ such that  $\|\widehat{f} - f_{j*}\|_{\infty} \leq \delta$. Together with the fact that $\bm{E}\big[  f_0(\bm{x}(\bm{s}_i)) \epsilon_i \big] = 0$, we have
\begin{align*}
&\bm{E}\big[\frac{2}{n}\sum_{i=1}^n\widehat{f}(\bm{x}(\bm{s}_i))\epsilon_i \big]\\
=& \bm{E}\big[\frac{2}{n}\sum_{i=1}^n\big(\widehat{f}(\bm{x}(\bm{s}_i))-f_{j*}(\bm{x}(\bm{s}_i)) + f_{j*}(\bm{x}(\bm{s}_i) - f_0(\bm{x}(\bm{s}_i))\big)\epsilon_i \big]\\
\le& \bm{E}\Big|\frac{2}{n}\sum_{i=1}^n\big(\widehat{f}(\bm{x}(\bm{s}_i))-f_{j*}(\bm{x}(\bm{s}_i))\big)\epsilon_i \Big| + \bm{E}\Big|\frac{2}{n}\sum_{i=1}^n\big(f_{j*}(\bm{x}(\bm{s}_i)) - f_0(\bm{x}(\bm{s}_i))\big)\epsilon_i \Big|\\
\le& \frac{2\delta}{n}\bm{E}\Big[\sum_{i=1}^n\big|\epsilon_i \big|\Big] + \frac{2}{n}\bm{E}\Big| \sum_{i=1}^n\big(f_{j*}(\bm{x}(\bm{s}_i)) - f_0(\bm{x}(\bm{s}_i))\big)\epsilon_i \Big| \\
\doteq& T_1 + T_2. \numberthis \label{eqn: bnd T_1 T_2}
\end{align*}
It is easy to see that $T_1 \le 2\delta(n^{-1}\tr\Gamma_n + \sigma)$.
For the second term $T_2$, notice that, conditionally on $\{\bm{x}(\bm{s}_1),  \ldots, \bm{x}(\bm{s}_n)\}$,
\[
\eta_j \doteq \frac{\sum_{i=1}^n\big(f_j(\bm{x}(\bm{s}_i)) - f_0(\bm{x}(\bm{s}_i))\big)\epsilon_i }{\sqrt{\bba_j^{\top}\Gamma_n \bba_j + n\sigma^2 \|f_j - f_0\|_n^2 }}
\]
follows $\mathcal{N}(0, 1)$ where $\bba_j = (f_j(\bm{x}(\bm{s}_1)) - f_0(\bm{x}(\bm{s}_1)), \ldots, f_j(\bm{x}(\bm{s}_n)) - f_0(\bm{x}(\bm{s}_n)))^{\top}$, $\| f \|_n^2 = n^{-1}\sum_{i=1}^nf(\bm{x}(\bm{s}_i))^2$.
From Lemma C.1 of \cite{schmidt-hieber},  $\bm{E}_{f _0}\big[ \max_{j=1, \ldots, \mathcal{N}}\eta_j^2 | \bm{x}(\bm{s}_1),  \ldots, \bm{x}(\bm{s}_n) \big] \le 3\log \mathcal{N} + 1$.  Consequently,
\begin{align*}
 T_2 =& \frac{2}{n} \bm{E}\Big|\eta_{j*}  \sqrt{\bba_{j*}^{\top}\Gamma_n \bba_{j*} + n\sigma^2 \|f_{j*} - f_0\|_n^2 } \Big|  \\
\le& \frac{2}{n} \bm{E}\Big( |\eta_{j*}|  \sqrt{(\tr(\Gamma_n^2) + n\sigma^2) \|f_{j*} - f_0\|_n^2 } \Big) \\
\le& \frac{2\sqrt{\tr(\Gamma_n^2) + n\sigma^2} }{n} \bm{E}\Big( |\eta_{j*}| (\|\hat f - f_0\|_n  + \delta)\Big)  \\
\le& \frac{2\sqrt{\tr(\Gamma_n^2) + n\sigma^2} }{n}\sqrt{3\log \mathcal{N} + 1}  \Big( \sqrt{R_n(\widehat{f}, f_0)}  + \delta \Big)
\end{align*}


Together with  (\ref{eqn: bnd R_n}) and  (\ref{eqn: bnd T_1 T_2}), we have
\begin{align*}
R_n(\widehat{f}, f_0)  \le& R_n(f, f_0) + \Delta_n  + 2\delta(n^{-1}\tr\Gamma_n + \sigma) +  \frac{2\sqrt{\tr(\Gamma_n^2) + n\sigma^2} }{n}\sqrt{3\log\mathcal{N} + 1}  \Big( \sqrt{R_n(\widehat{f}, f_0)}  + \delta \Big) .
\end{align*}

If $\log\mathcal{N}\le n$, then
\begin{align*}
R_n(\widehat{f}, f_0)  \le& R_n(f, f_0) + \Delta_n  + 2\delta\Big(n^{-1}\tr(\Gamma_n) + \sigma +   2\sqrt{n^{-1}\tr(\Gamma_n^2) + \sigma^2}\Big) \\
&+ \frac{2\sqrt{\tr(\Gamma_n^2) + n\sigma^2} }{n}\sqrt{3\log\mathcal{N} + 1}\sqrt{R_n(\widehat{f}, f_0)} .
\end{align*}
Applying the inequality (43) in \cite{schmidt-hieber}, we have, for any $0<\varepsilon\le 1$,
\begin{align*}
R_n(\widehat{f}, f_0)
\le& (1+\varepsilon)\left( R_n(f, f_0) + \Delta_n  + 2\delta\Big(n^{-1}\tr(\Gamma_n) + \sigma +   2\sqrt{n^{-1}\tr(\Gamma_n^2) + \sigma^2}\Big) \right) \\
&+ \frac{(1+\varepsilon)^2}{\varepsilon} \frac{1}{n^2}(3\log\mathcal{N} + 1)(\tr(\Gamma_n^2) + n\sigma^2) \\
\le& (1+\varepsilon)\left( R_n(f, f_0) + \Delta_n  + 2\delta\Big(n^{-1}\tr(\Gamma_n)  +   2\sqrt{n^{-1}\tr(\Gamma_n^2)} + 3\sigma \Big) \right) \\
&+ (1+\varepsilon) \frac{2F^2}{n\varepsilon}(3\log\mathcal{N} + 1)(n^{-1}\tr(\Gamma_n^2) + \sigma^2 +1).
\end{align*}

For $\log\mathcal{N}> n$,  $R_n(\widehat{f}, f_0) = \frac{1}{n}\sum_{i=1}^n\bm{E}_{f_0}\big[\big(\widehat{f}(\bm{x}(\bm{s}_i)) - f_0(\bm{x}(\bm{s}_i))\big)^2\big] \le 4F^2$ and
\begin{align*}
& (1+\varepsilon)\left( R_n(f, f_0) + \Delta_n  + 2\delta\Big(n^{-1}\tr(\Gamma_n)  +   2\sqrt{n^{-1}\tr(\Gamma_n^2)} + 3\sigma \Big) \right) \\
&+ (1+\varepsilon) \frac{2F^2}{n\varepsilon}(3\log\mathcal{N} + 1)(n^{-1}\tr(\Gamma_n^2) + \sigma^2+1) \\
>&   \frac{2F^2}{n}(3n+ 1)  > 6F^2.
\end{align*}
Thus,
\begin{align*}
R_n(\widehat{f}, f_0)  \le&  (1+\varepsilon)\left( R_n(f, f_0) + \Delta_n  + 2\delta\Big(n^{-1}\tr(\Gamma_n)  +  2\sqrt{n^{-1}\tr(\Gamma_n^2)} + 3\sigma \Big) \right) \\
&+ (1+\varepsilon) \frac{2F^2}{n\varepsilon}(3\log\mathcal{N} + 1)(n^{-1}\tr(\Gamma_n^2) + \sigma^2+1).
\end{align*}
Since the above inequality holds true for any $f\in\mathcal{F}$, we can prove the result by letting $f=\arginf_{\tilde{f}\in\mathcal{F}}R_n(\tilde{f}, f_0)$.
\end{proof}

\noindent\textbf{\emph{Proof of Theorem \ref{thm: oracle}}}: It follows from Lemma 5 and Remark 1 of \citet{schmidt-hieber} that
\begin{align*}
\log\mathcal{N} = \log\mathcal{N}(\delta, \mathcal{F}(L, \bm{p}, \tau, F), \| \cdot \|_{\infty})
\le& (1+\tau)\log(2^{5+2L}\delta^{-1}(L+1)\tau^{2L}d^2).
\end{align*}
Because $F\ge1$ and $0<\varepsilon\le1$,  we have
\begin{align*}
	R_n(\widehat{f}, f_0)  \lesssim&  (1+\varepsilon)\left( \inf_{\tilde{f} \in \mathcal{F}(L, \bm{p}, \tau, F)} \|\tilde{f}-f_0\|^2_\infty + \Delta_n(\widehat{f})  + \varsigma_{n, \varepsilon, \delta}  \right) ,
\end{align*}
where
\begin{align*}
\varsigma_{n, \varepsilon, \delta} \asymp& \frac{1}{\varepsilon} \left[ \delta  \Big(n^{-1}\tr(\Gamma_n)  +  2\sqrt{n^{-1}\tr(\Gamma_n^2)} + 3\sigma \Big) + \frac{\tau}{n} \left(\log(L/\delta) + L\log\tau \right) (n^{-1}\tr(\Gamma_n^2) + \sigma^2+1)      \right].
\end{align*}
$\square$

\subsection{Proof of Theorem \ref{thm.main}}

\begin{lemma} \label{lemma: approx}
For any $f: \mathbb{R}^{d} \to \mathbb{R} \in \mathcal{CS}(L_*, \bm{r}, \bm{\tilde{r}}, \bm{\beta},\bm{a}, \bm{b}, \bm{C})$,  $m\in\mathbb{Z}^{+}$, and $N \geq \max_{i=0,\ldots,L_*}(\beta_i+1)^{\tilde{r}_i} \vee (\tilde{C}_i+1)e^{\tilde{r}_i}$, there exists a neural network
\[
f_* \in \mathcal{F}(L, (d, 6\eta N,\ldots,6\eta N, 1), \sum_{i=0}^{L_*}r_{i+1}(\tau_i + 4), \infty),
\]
 such that
\[
\| f_* - f\|_{\infty} \leq C_{L_*}\prod_{l=0}^{L_*-1}(2C_l)^{\beta_{l+1}}\sum_{i=0}^{L_*}\left((2\tilde{C}_i+1)(1+\tilde{r}^2_i + \beta_i^2)6^{\tilde{r}}N2^{-m} + \tilde{C}_i3^{\beta_i}N^{-\frac{\beta_i}{\tilde{r}_i}}\right)^{\prod_{l=i+1}^{L_*}\beta_l\wedge1} ,
\]
where
\begin{align*}
\tilde{C}_i &= \sum_{k=0}^{i} C_k\frac{b_k - a_k}{b_{k+1}-a_{k+1}}, ~ i = 0,\ldots,L_*-1, \quad
\tilde{C}_{L_*} = \sum_{k=0}^{L_*}C_k\frac{b_k - a_k}{b_{k+1}-a_{k+1}}+b_{L_*} - a_{L_*}\\
L &= 3L_* + \sum_{i=0}^{L_*}L_i, \quad \text{with }
L_i = 8 + (m+5)(1+\ceil{\log_2(\tilde{r}_i\vee\beta_i)}),\\
\tau_i &\leq 141(\tilde{r}_i + \beta_i + 1)^{3+\tilde{r}_i}N(m+6), ~ i = 0,\ldots,L_*, \\
\eta &= \max_{i=0,\ldots, L_*}(r_{i+1}(\tilde{r}_i + \ceil{\beta_i})).
\end{align*}
\end{lemma}
\begin{proof}
By definition, we write $f(\bm{z})$ as
\begin{equation*}
     f(\bm{z})=\bm{g}_{L_*}\circ\ldots\circ \bm{g}_1 \circ \bm{g}_0(\bm{z}),\quad\quad \textrm{ for } \bm{z} \in [a_0, b_0]^{r_0} ,
\end{equation*}
where $\bm{g}_i=(g_{i,1},\ldots, g_{i,r_{i+1}})^\top: [a_i, b_i]^{r_i}\to [a_{i+1}, b_{i+1}]^{r_{i+1}}$ for some $|a_i|, |b_i|\leq C_i$ and the functions $g_{i,j}: [a_i, b_i]^{\tilde{r}_i} \to [a_{i+1}, b_{i+1}]$ are $(\beta_i, C_i)$-H\" older smooth and $r_{L_*+1} = 1$. For $i=0,\ldots, L_*-1$, the domain and range of $\bm{g}_i$ are $[a_i, b_i]^{r_i}$ and $[a_{i+1}, b_{i+1}]^{r_{i+1}}$, respectively.  First of all, we will rewrite $f$ as the composition of the functions $\bm{h}_i :=(h_{i,1},\ldots, h_{i,r_{i+1}})^\top$ whose domain and range are $[0, 1]^{r_i}$ and $[0, 1]^{r_{i+1}}$ which are constructed via linear transformation. That is, we define
\begin{align*}
\bm{h}_i(\bm{z}) &:= \frac{\bm{g}_i((b_i - a_i)\bm{z} - a_{i+1})}{b_{i+1} - a_{i+1}},\quad\quad \textrm{ for } \bm{z} \in [0, 1]^{r_i}, ~ i = 0,\ldots, L_*-1, \\
\bm{h}_{L_*}(\bm{z}) &:= \bm{g}_{L_*}((b_{L_*} - a_{L_*})\bm{z}+a_{L_*}),\quad\quad \textrm{ for } \bm{z} \in [0, 1]^{r_{L_*}}.
\end{align*}
Therefore the following equality holds
\begin{equation*}
     f(\bm{z})=\bm{g}_{L_*}\circ\ldots\circ \bm{g}_1 \circ \bm{g}_0(\bm{z})=\bm{h}_{L_*}\circ\ldots\circ \bm{h}_1 \circ \bm{h}_0(\frac{\bm{z}-a_0}{b_0 - a_0}),\quad\quad \textrm{ for } \bm{z} \in [a_0, b_0]^{r_0} .
\end{equation*}
Since $g_{i,j}: [a_i, b_i]^{\tilde{r}_i} \to [a_{i+1}, b_{i+1}]$ are all  $(\beta_i, C_i)$-H\" older smooth, it follows that $h_{i,j}: [0, 1]^{\tilde{r}_i} \to [0, 1]$ are all  $(\beta_i, \tilde{C}_i)$-H\" older smooth as well, where $\tilde{C}_i$ is a constant only depending on $\bm{a}, \bm{b}$, and $\bm{C}$, i.e., $\tilde{C}_i = \sum_{k=0}^{i}C_k\frac{b_k-a_k}{b_{k+1}-a_{k+1}}$ for $i = 0,\ldots,L_*-1$, and $\tilde{C}_{L_*} = \sum_{k=0}^{L_*}C_k\frac{b_k-a_k}{b_{k+1}-a_{k+1}}+b_{L_*} - a_{L_*}$.

By Theorem 5 of \citet{schmidt-hieber}, for any integer $m\geq1$ and $N \geq \max_{i=0,\ldots,L_*}(\beta_i+1)^{\tilde{r}_i} \vee (\tilde{C}_i+1)e^{\tilde{r}_i}$, there exists a network
\[
\tilde{h}_{i,j} \in \mathcal{F}(L_i, (\tilde{r}_i, 6(\tilde{r}_i + \ceil{\beta_i})N,\ldots,6(\tilde{r}_i + \ceil{\beta_i})N,1), \tau_i, \infty),
\]
with $L_i = 8 + (m+5)(1+\ceil{\log_2(\tilde{r}_i\vee\beta_i)})$ and $\tau_i\leq141(\tilde{r}_i + \beta_i + 1)^{3+\tilde{r}_i}N(m+6)$, such that
\[
\|\tilde{h}_{i,j} - h_{i,j}\|_\infty \leq (2\tilde{C}_i+1)(1+\tilde{r}^2_i + \beta_i^2)6^{\tilde{r}_i}N2^{-m} + \tilde{C}_i3^{\beta_i}N^{-\frac{\beta_i}{\tilde{r}_i}}.
\]
Note that the value of $\tilde{h}_{i,j}$ is $(-\infty, \infty)$. So we define $h^*_{i,j} := \sigma(-\sigma(-\tilde{h}_{i,j}+1)+1)$ by adding two more layers $\sigma(1-x)$ to restrict $h^*_{i,j}$ into the interval $[0, 1]$, where $\sigma(x) = \max(0, x)$. This introduces two more layers and four more parameters. By the fact that $h_{i,j} \in [0, 1]$, we have $h^*_{i,j}\in \mathcal{F}(L_i+2, (\tilde{r}_i, 6(\tilde{r}_i + \ceil{\beta_i})N,\ldots,6(\tilde{r}_i + \ceil{\beta_i})N,1), \tau_i + 4, \infty)$ and
\[
\|h^*_{i,j} - h_{i,j}\|_\infty \leq \|\tilde{h}_{i,j} - h_{i,j}\|_\infty \leq (2\tilde{C}_i+1)(1+\tilde{r}^2_i + \beta_i^2)6^{\tilde{r}_i}N2^{-m} + \tilde{C}_i3^{\beta_i}N^{-\frac{\beta_i}{\tilde{r}_i}}.
\]
We further parallelize all $(h^*_{i,j})_{j=1,\ldots,r_{i+1}}$ together, obtaining $\bm{h}^*_{i} := (h^*_{i, 1}, \ldots, h^*_{i, r_{i+1}})^\top \in \mathcal{F}(L_i+2, (r_i, 6r_{i+1}(\tilde{r}_i + \ceil{\beta_i})N,\ldots,6r_{i+1}(\tilde{r}_i + \ceil{\beta_i})N,r_{i+1}), r_{i+1}(\tau_i + 4), \infty)$.
Moreover, we construct the composite network $f_* := \bm{h}^*_{L_*}\circ\ldots\circ\bm{h}^*_1\circ\bm{h}^*_0 \in \mathcal{F}(3L_* + \sum_{i=0}^{L_*}L_i, (r_0, 6\eta N,\ldots,6\eta N, 1), \sum_{i=0}^{L_*}r_{i+1}(\tau_i + 4), \infty)$, where $\eta = \max_{i=0,\ldots, L_*}(r_{i+1}(\tilde{r}_i + \ceil{\beta_i}))$.

By Lemma 3 in \citet{schmidt-hieber},
\begin{align*}
\|f - f_* \|_\infty =& \|\bm{h}_{L_*}\circ\ldots\circ \bm{h}_1 \circ \bm{h}_0 -  \bm{h}^*_{L_*}\circ\ldots\circ\bm{h}^*_1\circ\bm{h}^*_0\|_\infty\\
\leq& C_{L_*}\prod_{l=0}^{L_*-1}(2C_l)^{\beta_{l+1}}\sum_{i=0}^{L_*}\||\bm{h}_i - \bm{h}^*_i|_\infty\|_\infty^{\prod_{l=i+1}^{L_*}\beta_l\wedge1}\\
\leq& C_{L_*}\prod_{l=0}^{L_*-1}(2C_l)^{\beta_{l+1}}\sum_{i=0}^{L_*}\left((2\tilde{C}_i+1)(1+\tilde{r}^2_i + \beta_i^2)6^{\tilde{r}}N2^{-m} + \tilde{C}_i3^{\beta_i}N^{-\frac{\beta_i}{\tilde{r}_i}}\right)^{\prod_{l=i+1}^{L_*}\beta_l\wedge1}\\
\leq& C_{L_*}\prod_{l=0}^{L_*-1}(2C_l)^{\beta_{l+1}}\sum_{i=0}^{L_*}((2\tilde{C}_i+1)(1+\tilde{r}^2_i + \beta_i^2)6^{\tilde{r}}N2^{-m})^{\prod_{l=i+1}^{L_*}\beta_l\wedge1} \\
&+C_{L_*}\prod_{l=0}^{L_*-1}(2C_l)^{\beta_{l+1}}\sum_{i=0}^{L_*}(\tilde{C}_i3^{\beta_i}N^{-\frac{\beta_i}{\tilde{r}_i}})^{\prod_{l=i+1}^{L_*}\beta_l\wedge1}.
\end{align*}
\end{proof}

\noindent\textbf{\emph{Proof of Theorem \ref{thm.main}}}: By Theorem \ref{thm: oracle} with $\delta = n^{-2}$ and $\varepsilon=1$, it follows that
\[
R_n(\widehat{f}_{\textrm{local}}, f_0) \lesssim \inf_{\tilde{f} \in \mathcal{F}(L, \bm{p}, \tau, F)} \|\tilde f-f_0\|^2_\infty  + \frac{(\tr(\Gamma_n^2)+n)\tau(\log(Ln^2)+L\log{\tau})}{n^2} + \Delta_n(\widehat{f}_{\textrm{local}}).
\]
Next, we need to analyze the first term. Since $f_{0} \in \mathcal{CS}(L_*, \bm{r}, \bm{\tilde{r}}, \bm{\beta},\bm{a}, \bm{b}, \bm{C})$, by Lemma \ref{lemma: approx}, for any $m>0$, there exists a neural network
\[
f_{*} \in \mathcal{F}(L, (d, N,\ldots, N, 1), \tau, \infty),
\]
with $L \asymp m, N \geq 6\eta\max_{i=0,\ldots,L_*}(\beta_i+1)^{\tilde{r}_i} \vee (\tilde{C}_i+1)e^{\tilde{r}_i}, \eta = \max_{i=0,\ldots, L_*}(r_{i+1}(\tilde{r}_i + \ceil{\beta_i})), \tau\lesssim mN,$ such that
\begin{align} \label{eq:proof_thm2}
\| f_{*} - f_{0}\|_{\infty} &\lesssim \sum_{i=0}^{L_*}(N2^{-m})^{\prod_{l=i+1}^{L_*}\beta_l\wedge1} + (N^{-\frac{\beta_i}{\tilde{r}_i}})^{\prod_{l=i+1}^{L_*}\beta_l\wedge1}\nonumber\\
&\lesssim \sum_{i=0}^{L_*}(N2^{-m})^{\prod_{l=i+1}^{L_*}\beta_l\wedge1} + N^{-\frac{\beta_i^*}{\tilde{r}_i}}\nonumber\\
&\lesssim (N2^{-m})^{\prod_{l=1}^{L_*}\beta_l\wedge1} + N^{-\frac{\beta^*}{r^*}},
\end{align}
where recall that $\beta^*=\beta_{i^*}^*$ and $r^*=\tilde{r}_{i^*}$. For simplicity, we let $\bm{p} = (d, N,\ldots, N, 1)$. This means there exists a sequence of networks $(f_n)_n$ such that for all sufficiently large $n$, $\| f_n - f_{0}\|_{\infty} \lesssim (N2^{-m})^{\prod_{l=1}^{L_*}\beta_l\wedge1} + N^{-\frac{\beta^*}{r^*}}$ and $f_n \in \mathcal{F}(L, \bm{p}, \tau, \infty).$ Next define $\grave{f}:=f_n(\|f_{0}\|_\infty/\|f_{n}\|_\infty\wedge1) \in \mathcal{F}(L, \bm{p}, \tau, F), F \geq \max_{i=0,\ldots,L^*}(C_i, 1)$, and it is obvious that $\|\grave{f} - f_{0}\|_{\infty}\lesssim (N2^{-m})^{\prod_{l=1}^{L_*}\beta_l\wedge1} + N^{-\frac{\beta^*}{r^*}}$. Then it follows that
\begin{equation}
\inf_{\tilde{f} \in \mathcal{F}(L, \bm{p}, \tau, F)}\|\tilde{f} - f_{0}\|_\infty \lesssim\|\grave{f} - f_{0}\|_{\infty}\lesssim (N2^{-m})^{\prod_{l=1}^{L_*}\beta_l\wedge1} + N^{-\frac{\beta^*}{r^*}}.
\end{equation}

By combining (\ref{eq:proof_thm2}) and the fact that $\tau\lesssim LN$, the proof is completed.
$\square$
\end{document}